\documentclass{article}

\usepackage{microtype}
\usepackage{graphicx}
\usepackage{subcaption}
\usepackage{booktabs} 

\usepackage{hyperref}

\usepackage[preprint]{icml2026}

\usepackage{amsmath}
\usepackage{amssymb}
\usepackage{mathtools}
\usepackage{amsthm}
\usepackage{multirow}

\usepackage[capitalize,noabbrev]{cleveref}

\theoremstyle{plain}
\newtheorem{theorem}{Theorem}[section]

\newtheorem{lemma}[theorem]{Lemma}

\theoremstyle{definition}
\newtheorem{definition}[theorem]{Definition}

\theoremstyle{remark}
\newtheorem{remark}[theorem]{Remark}

\usepackage{xspace}
\usepackage{amsmath}
\newcommand{\method}{DAF\xspace}
\usepackage[textsize=tiny]{todonotes}

\icmltitlerunning{Submission and Formatting Instructions for ICML 2026}

\begin{document}

\twocolumn[
  \icmltitle{Dynamical Adapter Fusion: Constructing A Global Adapter for Pre-Trained Model-based Class-Incremental Learning}




  \begin{icmlauthorlist}
    \icmlauthor{Ruiqi Liu}{yyy,comp}
    \icmlauthor{Boyu Diao}{yyy,comp}
    \icmlauthor{Zijia An}{yyy,comp}
    \icmlauthor{Zhulin An}{yyy,comp}
    \icmlauthor{Fei Wang}{yyy,comp}
    \icmlauthor{Yongjun Xu}{yyy,comp}
  \end{icmlauthorlist}

  \icmlaffiliation{yyy}{Institute of Computing Technology, Chinese Academy of Sciences, Beijing, China}
  \icmlaffiliation{comp}{University of Chinese Academy of Sciences, Beijing, China}

  \icmlcorrespondingauthor{ Boyu Diao}{diaoboyu2012@ict.ac.cn}

  \icmlkeywords{Machine Learning, ICML}

  \vskip 0.3in
]



\printAffiliationsAndNotice{}  

\begin{abstract}
Class-Incremental Learning (CIL) requires models to continuously acquire new classes without forgetting previously learned ones. A dominant paradigm involves freezing a pre-trained model and training lightweight, task-specific adapters. However, maintaining task-specific parameters hinders knowledge transfer and incurs high retrieval costs, while naive parameter fusion often leads to destructive interference and catastrophic forgetting. To address these challenges, we propose Dynamical Adapter Fusion (DAF) to construct a single robust global adapter. Grounded in the PAC-Bayes theorem, we derive a fusion mechanism that explicitly integrates three components: the optimized task-specific adapter parameters, the previous global adapter parameters, and the initialization parameters. We utilize the Taylor expansion of the loss function to derive the optimal fusion coefficients, dynamically achieving the best balance between stability and plasticity. Furthermore, we propose a Robust Initialization strategy to effectively capture global knowledge patterns. Experiments on multiple CIL benchmarks demonstrate that DAF achieves state-of-the-art (SOTA) performance.
\end{abstract}
\section{Introduction}
\label{sec:intro}
\begin{figure}[t]
    \centering
    \includegraphics[width=0.38\textwidth]{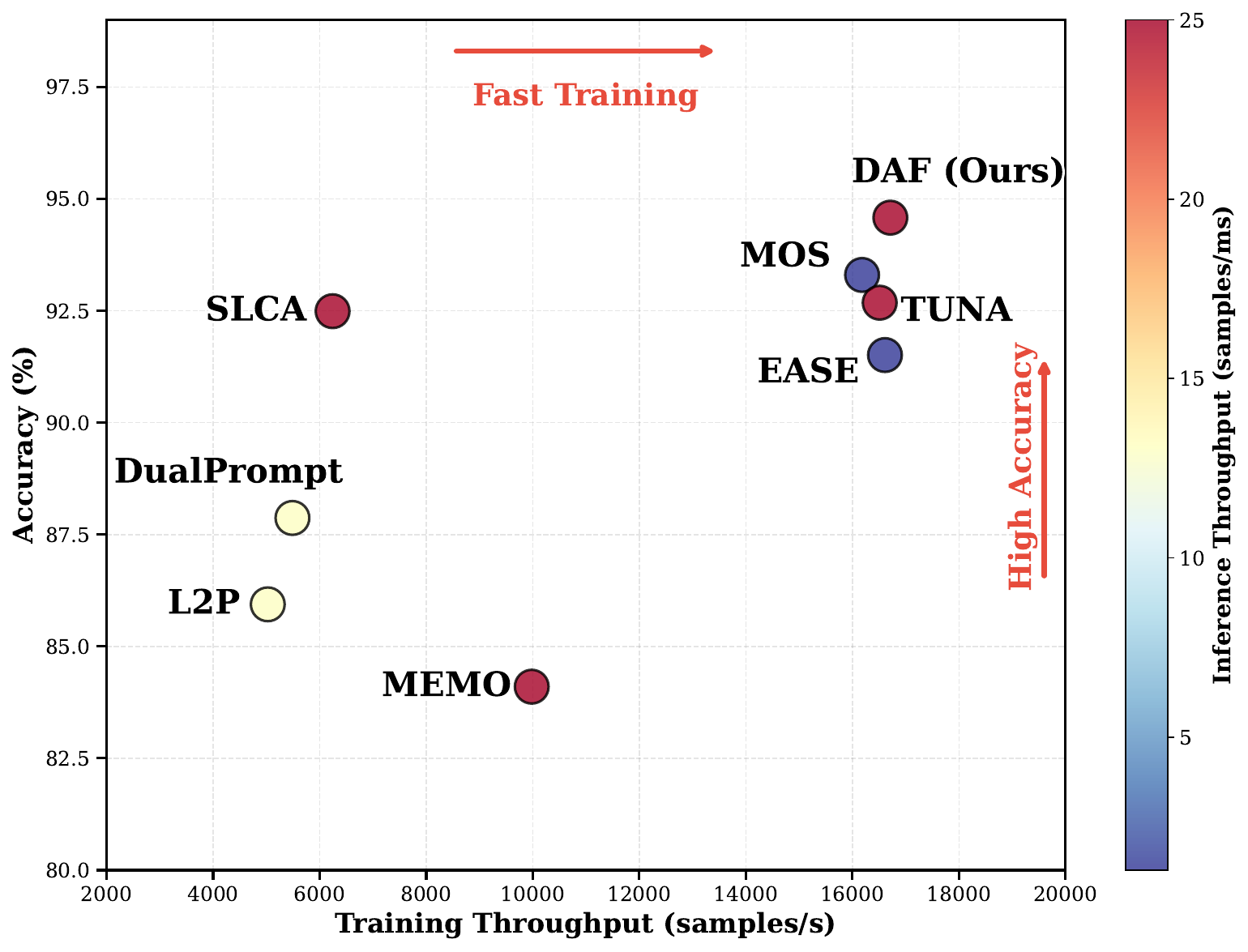}
    \caption{
        Efficiency and performance evaluation on CIFAR-100.
        We plot Average Accuracy ($\bar{\mathcal{A}}$) against Training Throughput (samples/s). 
        The color of each point represents the Inference Throughput (samples/ms). 
        The ideal method occupies the top-right corner with a dark red color (highest inference throughput).
        TUNA~\cite{wang2025integrating} is evaluated here using its single global adapter.
        Our DAF achieves this optimal balance, significantly outperforming complex methods in both speed and precision.
    }
    \label{fig:teaser}
\end{figure}
The objective of Class-Incremental Learning is to sequentially assimilate novel categories while preventing the erasure of historical knowledge, a phenomenon known as catastrophic forgetting~\cite{kirkpatrick2017overcoming,mccloskey1989catastrophic}. Beyond the primary struggle to mitigate forgetting, the widespread application of CIL in resource-limited environments, including embodied AI and edge computing, has made computational efficiency an equally urgent research imperative~\cite{soltoggio2024collective, liu2024resource}. Generally, prevailing methods can be bisected into two streams: training from scratch and exploiting Pre-trained Models (PTMs). Traditional from-scratch methods, which rely on regularization, dynamic architectures, or data replay~\cite{li2017learning, serra2018overcoming, rebuffi2017learning}, frequently face intrinsic bottlenecks regarding data privacy, scalability, and performance saturation~\cite{wang2022dualprompt}. Accordingly, the field has witnessed a paradigm shift toward PTM-based frameworks~\cite{zhou2024revisiting}. By capitalizing on the robust, transferable features acquired from extensive pre-training, these methods consistently yield SOTA results. Within this domain, the dominant strategy involves locking the PTM backbone and exclusively employing Parameter-Efficient Fine-Tuning (PEFT) techniques, such as Prompt Tuning~\cite{wang2022learning, smith2023coda} or Adapter Tuning~\cite{sun2025mos}. This paradigm has become the standard for modern CIL, offering an optimal compromise between minimal resource consumption and superior classification accuracy.

A common strategy trains separate adapters for each task~\cite{zhou2024expandable}, requiring retrieval or ensembling during inference~\cite{smith2023coda,zhou2024expandable}. While this reduces task interference, it incurs high retrieval costs and risks performance degradation from incorrect selection. Moreover, memory usage grows with the number of tasks. Consequently, we aim to construct a single global adapter without storing historical modules. However, existing global methods often retain this storage overhead~\cite{wang2025integrating, huai2025cl}, while simple fusion techniques like Exponential Moving Average (EMA) fail to capture sufficient global knowledge~\cite{qiaolarge}.

To this end, we propose Dynamical Adapter Fusion, a novel framework designed to construct a single robust global adapter for PTM-based CIL. We adopt the PAC-Bayes theorem to provide insight into the dilemma between effective learning and retaining previous knowledge. Our theoretical analysis reveals that the upper bound of the expected risk can be divided into three distinct parts: empirical risk, generalization regularization, and forgetting regularization. Guided by this analysis, upon the completion of each task, we fuse three components: the optimized task-specific adapter parameters, the previous global adapter parameters, and the initialization parameters of the current task-specific adapter. We utilize the Taylor expansion of the loss function to establish an ideal optimization objective for global adapter fusion that explicitly accounts for these three terms. By employing the method of Lagrange Multipliers to solve this objective, we derive the optimal fusion coefficients to adaptively achieve the best balance between stability and plasticity. Furthermore, we propose a Robust Initialization strategy that recursively averages historical adapters as an informative prior to capture global knowledge patterns and enhance robustness. Consequently, our entire framework operates without the need for storing historical adapters, resulting in the exceptional inference throughput and accuracy balance illustrated in Figure~\ref{fig:teaser}. In summary, our main contributions are as follows:

\begin{itemize}
    \item We propose DAF, a framework developed from a PAC-Bayesian perspective to address the plasticity-stability dilemma. By formulating the fusion process as a constrained optimization problem and solving it via Lagrange Multipliers, we construct a single global adapter that dynamically integrates new knowledge, thereby eliminating the retrieval costs and storage overhead associated with maintaining historical adapters.
    
    \item We propose a memory-efficient robust initialization strategy that recursively averages historical adapters to serve as an informative prior. This strategy effectively anchors the learning process to global knowledge patterns, significantly enhancing model stability without accessing past data.
    
    \item We conduct extensive experiments on major CIL benchmarks, where DAF consistently outperforms SOTA methods with accuracy gains of up to 2.75\%.
\end{itemize}

\section{Related Work}
\label{sec:related_work}

CIL methods can be broadly categorized into two primary paradigms: training from scratch and leveraging PTMs.

\subsection{Training from Scratch}
The conventional CIL paradigm involves training models sequentially starting from random initialization. These methods are generally classified into three major categories. \textbf{Regularization-based} methods~\cite{yu2020semantic,serra2018overcoming,li2017learning,yang2024clip,chaudhry2018riemannian,aljundi2018memory} add penalty terms to the loss function to constrain the drift of parameters essential for previous tasks. \textbf{Architecture-based} methods~\cite{ebrahimi2020adversarial,ke2020continual,loo2020generalized,wang2020learn,zhao2022deep} prevent interference by allocating distinct model parameters or sub-networks to each incoming task. \textbf{Rehearsal-based} methods~\cite{aljundi2019gradient,buzzega2020dark,cha2021co2l,chaudhry2021using,chaudhry2018efficient} maintain a small buffer of historical exemplars to approximate joint training with past data. However, despite their fundamental importance, these from-scratch methods frequently encounter challenges related to scalability, data privacy, and performance saturation, prompting a transition toward the PTM-based paradigm~\cite{ smith2021always, wang2022dualprompt, li2024fcs}.

\subsection{Pre-Trained Model-based CIL}

To mitigate the substantial computational burden of training from scratch, recent CIL research heavily utilizes PTMs combined with PEFT. This method typically freezes the massive PTM backbone and optimizes only a small set of incremental parameters, making CIL feasible for resource-constrained scenarios. Methods in this domain are diverse. Prompt-learning methods may optimize a shared pool of prompts~\cite{wang2022learning}, assign disjoint prompts for specific functions~\cite{wang2022dualprompt}, or dynamically compose prompts~\cite{smith2023coda}. Alternative strategies include fine-tuning specific backbone blocks~\cite{zhang2023slca}, combining PEFT with linear classifiers~\cite{zhou2024revisiting, mcdonnell2024ranpac}, or appending task-specific modules~\cite{zhou2024expandable, sun2025mos}.

A critical challenge within this paradigm lies in effectively integrating task-specific parameters. Recent works have explored sophisticated retrieval and fusion mechanisms. For instance, MOS~\cite{sun2025mos} employs an iterative retrieval mechanism to select adapters at test time, while TUNA~\cite{wang2025integrating} uses entropy-based retrieval combined with max-fusion to form a global adapter. In contrast to these retrieval-dependent or heuristic strategies, our proposed DAF theoretically derives optimal fusion weights via the PAC-Bayes framework, constructing a single robust global adapter that balances stability and plasticity without inference-time retrieval.

\section{Methodology}
\label{sec:methodology}

\subsection{Preliminaries}

\paragraph{Class-Incremental Learning.}
Formally, CIL involves learning a sequence of $T$ tasks $\{\mathcal{D}_t\}_{t=1}^T$. Each dataset $\mathcal{D}_t = \{(x_i, y_i)\}_{i=1}^{N_t}$ contains samples $x_i$ with labels $y_i \in \mathcal{C}_t$. The label sets are disjoint, satisfying $\mathcal{C}_t \cap \mathcal{C}_{t'} = \emptyset$ for $t \neq t'$. Following the exemplar-free setting~\cite{wang2022learning,wang2022dualprompt,zhou2024revisiting}, we train the model using only the current data $\mathcal{D}_t$ without accessing past samples. After each task, the model is evaluated on the cumulative label set $\mathcal{Y}_t = \bigcup_{i=1}^{t} \mathcal{C}_i$.

\paragraph{PTM-based CIL via Task-Specific Adapters.}
We adopt a PTM-based framework (e.g., ViT~\cite{dosovitskiy2020image}) where the backbone is frozen and only lightweight adapters are trained. Following~\cite{zhou2024revisiting,wang2025integrating}, bottleneck adapters are inserted parallel to the MLP layers. For input $\mathbf{x}_i \in \mathbb{R}^d$, the output $\mathbf{x}_o$ is computed as:
\begin{align} \label{eq:adapter_forward}
	\mathbf{x}_o = \text{MLP}(\mathbf{x}_i) + \text{ReLU}(\mathbf{x}_i W_{\text{down}})W_{\text{up}},
\end{align}
where $W_{\text{down}} \in \mathbb{R}^{d \times r}$ and $W_{\text{up}} \in \mathbb{R}^{r \times d}$ are projection matrices with rank $r \ll d$. For each task $t$, a specific adapter set $\mathcal{A}_t$ is trained.
To avoid memory expansion, we employ a fusion framework where $\mathcal{A}_t$ is integrated into a single global adapter $\mathcal{A}_g$ after training. Thus, we only maintain the global $\mathcal{A}_g$ and the current temporary $\mathcal{A}_t$. During inference, the model exclusively uses $\mathcal{A}_g$, eliminating the need for task-ID retrieval mechanisms~\cite{sun2025mos,wang2025integrating}. A unified classifier is constructed via prototype-based alignment (see Appendix~\ref{app:alignment}).

\subsection{Theoretical Analysis} \label{sec:theory}
We analyze the generalization error using PAC-Bayes theorem~\cite{alquier2016properties, wang2024forgetting}. Our method uses a task-specific adapter and a global adapter. We define the model space as $\mathcal{H} = \mathcal{H}_s \times \mathcal{H}_g$, where $\mathcal{H}_s$ and $\mathcal{H}_g$ are parameter spaces for the task-specific adapter and global adapter. A model is $h=(h_s, h_g)$. Let $\mathcal{X}$ be the data space. Let loss $\ell: \mathcal{H} \times \mathcal{X} \rightarrow \mathbb{R}^+$ be bounded by $K$. For task $t$, the prior $P_t = P_{\text{init}} \otimes Q_{t-1}^g$ combines fixed initialization $P_{\text{init}}$ and previous global posterior $Q_{t-1}^g$. The posterior $Q_t = Q_t^s \otimes Q_t^g$ combines the learned task-specific distribution $Q_t^s$ and the updated global distribution $Q_t^g$. We define expected risk $\mathcal{R} = \mathbb{E}_{h \sim Q_{t}}\left[ \mathbb{E}_{x \sim \mu_t}[\ell(h,x)] \right]$ and empirical risk $\hat{\mathcal{R}} = \frac{1}{N_t} \sum_{j=1}^{N_t} \mathbb{E}_{h \sim Q_{t}}\left[ \ell(h,x_j) \right]$, where $N_t$ is the sample size.

\begin{theorem}
  \label{th: split pac bayes main}
  Let $\mathcal{D}_t=(x_1,...,x_{N_t})$ be an iid set sampled from distribution $\mu_t$ for task $t$. Let prior $P_t = P_{\text{init}} \otimes Q_{t-1}^g$ and posterior $Q_t = Q_t^s \otimes Q_t^g$ be defined on $\mathcal{H}$. For any loss $\ell$ bounded by $K$, $\lambda>0$, and $\delta\in [0,1]$, the following inequality holds with probability $1-\delta$:
  \begin{align}
    \mathcal{R} \leq & \underbrace{\hat{\mathcal{R}}}_{\text{Empirical Risk}} + \underbrace{\frac{\operatorname{KL}(Q_t^s \| P_{\text{init}})}{\lambda}}_{\text{Generalization Regularization}} + \underbrace{\frac{\operatorname{KL}(Q_t^g \| Q_{t-1}^g)}{\lambda}}_{\text{Forgetting Regularization}} \nonumber \\
    &+ \frac{\log(1/\delta)}{\lambda} + \frac{\lambda K^2}{2N_t}.
  \end{align}
\end{theorem}

The complete proof is in Appendix \ref{proof_modified}.
The bound consists of three key components:
(1) \textbf{Empirical Risk} ($\hat{\mathcal{R}}$): Ensures plasticity by fitting current data.
(2) \textbf{Generalization Regularization} ($\operatorname{KL}(Q_t^s \| P_{\text{init}})$): Constrains task-specific parameters to the robust initialization $P_{\text{init}}$ to prevent overfitting.
(3) \textbf{Forgetting Regularization} ($\operatorname{KL}(Q_t^g \| Q_{t-1}^g)$): Constrains the global adapter to stay close to previous knowledge $Q_{t-1}^g$, ensuring stability.
Thus, an optimal fusion mechanism must balance minimizing current loss while restricting deviations from both the initialization and the previous global state.

\begin{figure*}[t]
    \centering
    \includegraphics[width=0.9\textwidth]{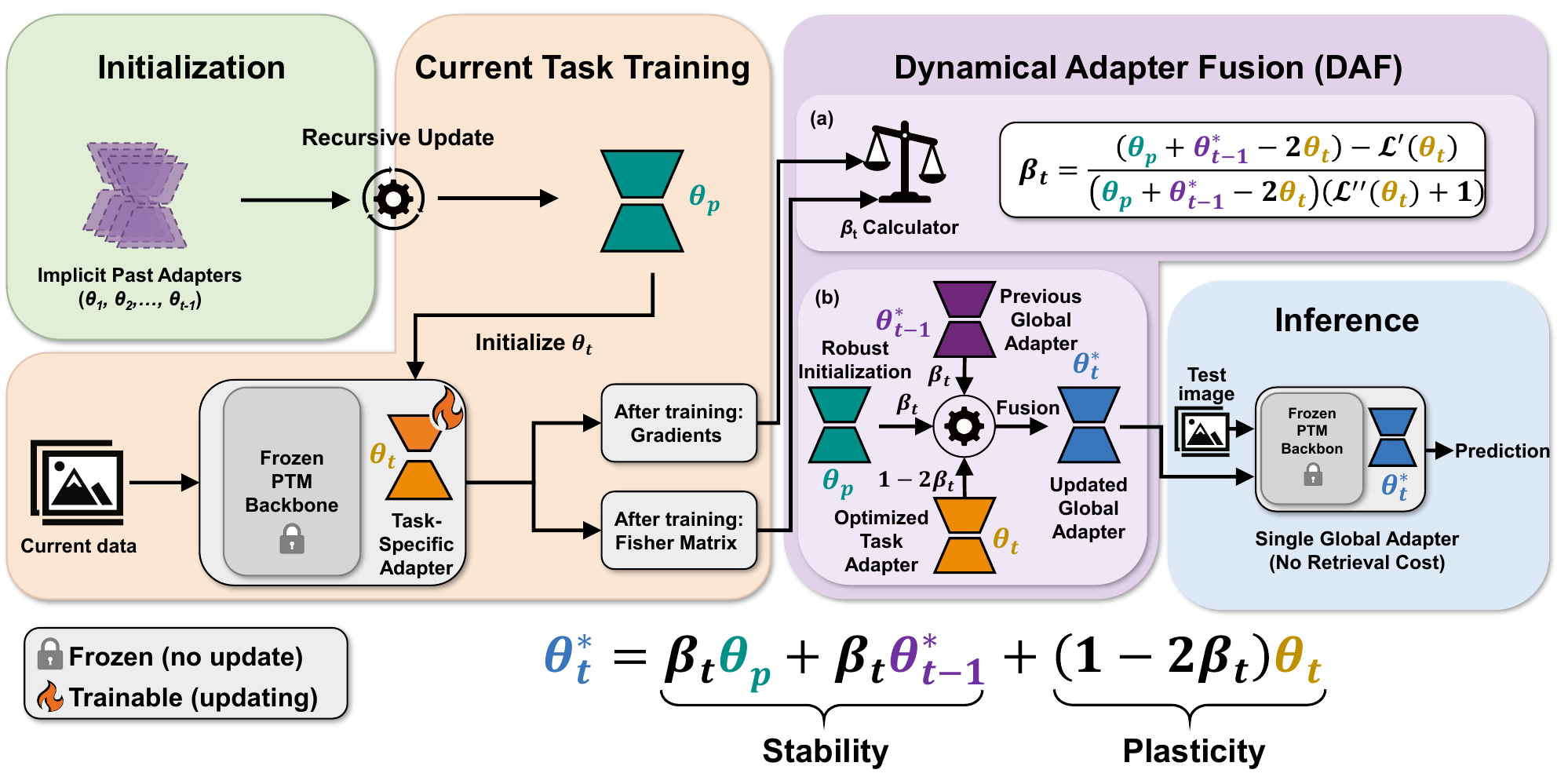} 
    \caption{Overview of the proposed DAF framework. The pipeline consists of four distinct phases: 
    \textbf{(1) Initialization:} We construct a \textbf{Robust Initialization} $\theta_p$ via a recursive update of implicit past adapters; 
    \textbf{(2) Current Task Training:} The \textbf{Optimized Task Adapter} $\theta_t$ is learned on new data, and statistics (Gradients, Fisher Matrix) are computed; 
    \textbf{(3) Dynamical Adapter Fusion :} This core module operates in two steps: (a) the \textbf{$\beta_t$ Calculator} derives the optimal coefficient using the Lagrangian solution, and (b) the \textbf{Fusion} process integrates the Robust Initialization $\theta_p$, the Previous Global Adapter $\theta_{t-1}^*$, and the Optimized Task Adapter $\theta_t$ to form the updated model; 
    \textbf{(4) Inference:} The resulting Single Global Adapter $\theta_t^*$ performs efficient, retrieval-free inference on all learned classes.}
    \label{fig:framework}
\end{figure*}
\subsection{Dynamical Adapter Fusion Method}

Guided by the PAC-Bayes analysis in Section \ref{sec:theory}, we propose Dynamical Adapter Fusion, as shown in Figure~\ref{fig:framework}. Unlike expansion-based methods that rely on dynamic architectures, DAF maintains a fixed model capacity. Instead, it dynamically optimizes the fusion coefficients after each task to adapt to the evolving data distribution.
For each task $t$, we initialize a task adapter with parameters $\theta_{p}$ and optimize it on the current data to obtain $\theta_{t}$. To ensure memory efficiency, we retain only a single global adapter $\theta_{t}^{*}$ by fusing the learned knowledge. Motivated by our theoretical decomposition, which splits the objective into empirical risk, generalization regularization, and forgetting regularization, we construct $\theta_{t}^{*}$ by explicitly integrating the three corresponding parameter states: the optimized $\theta_{t}$, the initialization $\theta_{p}$, and the previous global state $\theta_{t-1}^{*}$. We formulate this fusion as:
\begin{equation} \label{eq:fusion}
    \theta_{t}^{*} = \beta_{t}\theta_{p} + \beta_{t}\theta_{t-1}^{*} + \left(1 - 2\beta_{t}\right)\theta_{t},
\end{equation}
where $\theta$ denotes the individual parameters within the adapters. Here, the element-wise fusion coefficient $\beta_{t}$ balances the trade-off among these competing objectives. Instead of using suboptimal fixed values, we derive the optimal $\beta_t$ by minimizing the bounds defined in Theorem \ref{th: split pac bayes main}.

For the \textit{Empirical Risk}, we treat the optimized task adapter $\theta_{t}$ as a proxy target to avoid expensive retraining. We thus aim to minimize the loss gap $\mathcal{L}(\theta_{t}^{*}) - \mathcal{L}(\theta_{t})$. Regarding regularization, Theorem \ref{th: split pac bayes main} suggests constraining the deviations of $\theta_{t}$ from $\theta_{p}$ (Generalization Regularization) and $\theta_{t}^{*}$ from $\theta_{t-1}^{*}$ (Forgetting Regularization). To preserve plasticity, we enforce the generalization constraint during fusion rather than training. Specifically, we incorporate the deviation $\theta_{t} - \theta_{p}$ into the fusion objective to subtract excessive parameter shift, realigning $\theta_{t}^{*}$ with the robust initialization. We introduce $\Delta\theta$ to capture the aggregate parameter shift required by these stability constraints:
\begin{equation} \label{eq:constraint}
    \Delta\theta = \theta_{t}^{*} - \theta_{t-1}^{*} + \theta_{t} - \theta_{p}.
\end{equation}
By combining the risk approximation and the total parameter shift, we formulate the unified optimization problem as:
\begin{align} \label{eq:optimization_problem}
    \text{min}\quad &  \mathcal{L}(\theta_{t}^{*}) - \mathcal{L}(\theta_{t}) + \frac{1}{2} \left( \theta_{t}^{*} - \theta_{t-1}^{*} + \theta_{t} - \theta_{p} \right)^{2},  \nonumber\\
    \textit{s.t.} \quad & \Delta\theta + \theta_{t-1}^{*} + \theta_{p} - \theta_{t} - \theta_{t}^{*} = 0.
\end{align}

To solve the constrained minimization problem efficiently, we employ the method of Lagrange multipliers. This allows us to incorporate the stability constraint directly into the objective function by introducing a Lagrange multiplier $\lambda$. The Lagrangian function $F$ is defined as:
\begin{align} \label{eq:lagrangian}
    F &= \mathcal{L}(\theta_{t}^{*}) - \mathcal{L}(\theta_{t}) + \frac{1}{2}\left(\theta_{t}^{*} - \theta_{t-1}^{*} + \theta_{t} - \theta_{p}\right)^{2} \nonumber \\ &+\lambda\left(\Delta\theta + \theta_{t-1}^{*} + \theta_{p} - \theta_{t} - \theta_{t}^{*}\right).
\end{align}

To facilitate the derivation, we approximate the loss difference using a second-order Taylor expansion around the optimized parameter $\theta_{t}$. Omitting the high-order infinitesimal terms, we have:
\begin{equation} \label{eq:taylor}
    \mathcal{L}(\theta_{t}^{*}) - \mathcal{L}(\theta_{t}) \approx \mathcal{L}^{\prime}(\theta_{t})(\theta_{t}^{*} - \theta_{t}) + \frac{\mathcal{L}^{\prime\prime}(\theta_{t})}{2}(\theta_{t}^{*} - \theta_{t})^{2}.
\end{equation}

Furthermore, combining the fusion mechanism defined in Eq. \eqref{eq:fusion} with the definition of $\Delta\theta$, we can express the deviation $\theta_{t}^{*} - \theta_{t}$ in terms of the relaxation factor $\Delta\theta$:
\begin{equation} \label{eq:theta_diff}
    \theta_{t}^{*} - \theta_{t} = \frac{\beta_{t}}{\beta_{t}-1}\Delta\theta.
\end{equation}      
Please refer to Appendix~\ref{app:theta_diff_proof} for detailed demonstrations of this relationship.

Next, we substitute Eq. \eqref{eq:fusion} into the constraint term of the Lagrangian. This transforms the stability constraint equation into a form dependent on $\beta_t$:
\begin{align} \label{eq:constraint_trans}
    \theta_{t-1}^{*} + \theta_{p} - \theta_{t} - \theta_{t}^{*} = (1-\beta_{t})(\theta_{p} + \theta_{t-1}^{*} - 2\theta_{t}).
\end{align}

Substituting Eq. \eqref{eq:taylor}, Eq. \eqref{eq:theta_diff}, and Eq. \eqref{eq:constraint_trans} back into Eq. \eqref{eq:lagrangian}, we rewrite the Lagrangian $F$ entirely in terms of $\beta_t$ and $\Delta\theta$:
\begin{align} \label{eq:lagrangian_sub}
    F = &\mathcal{L}^{\prime}(\theta_{t})\frac{\beta_{t}}{\beta_{t}-1}\Delta\theta + \frac{\mathcal{L}^{\prime\prime}(\theta_{t})}{2}\left(\frac{\beta_{t}}{\beta_{t}-1}\Delta\theta\right)^{2} + \frac{1}{2}\Delta\theta^{2} \nonumber \\
    &+ \lambda\left[\Delta\theta + (1-\beta_{t})(\theta_{p} + \theta_{t-1}^{*} - 2\theta_{t})\right].
\end{align}

To find the optimal solution, we first take the derivative of the Lagrangian with respect to $\beta_{t}$ and set it to zero. This condition is essential for locating the stationary point of the function:
\begin{align} \label{eq:partial_beta}
    \frac{\partial F}{\partial\beta_{t}} &= -\frac{1}{(\beta_{t}-1)^{2}}\mathcal{L}^{\prime}(\theta_{t})\Delta\theta - \frac{\beta_{t}}{(\beta_{t}-1)^{3}}\mathcal{L}^{\prime\prime}(\theta_{t})\Delta\theta^{2}  \nonumber\\ &- \lambda(\theta_{p} + \theta_{t-1}^{*} - 2\theta_{t}) = 0.
\end{align}

By solving these equations, we obtain a feasible solution for $\beta_{t}$ that minimizes the objective function while satisfying the constraints:
\begin{equation} \label{eq:optimal_beta}
    \beta_{t} = \frac{(\theta_{p} + \theta_{t-1}^{*} - 2\theta_{t}) - \mathcal{L}^{\prime}(\theta_{t})}{(\theta_{p} + \theta_{t-1}^{*} - 2\theta_{t})(\mathcal{L}^{\prime\prime}(\theta_{t}) + 1)}.
\end{equation}
Please refer to Appendix~\ref{app:beta_proof} for the detailed deduction of this solution.

\subsection{Optimization Approximation} \label{sec:approx}

To implement the optimal solution in Eq. \eqref{eq:optimal_beta}, we require $\beta_{t} \in (0, 0.5)$ to balance stability and plasticity. The approximation $\beta_{t} \approx \frac{1}{\mathcal{L}^{\prime\prime}(\theta_{t}) + 1}$ implies that $\mathcal{L}^{\prime\prime}(\theta_{t}) > 1$ is necessary for $\beta_{t} < 0.5$.

We approximate the Hessian diagonal using the Fisher Information Matrix diagonal~\cite{kirkpatrick2017overcoming,zhou2025ferret}, denoted as $F$, computed on the current task data. Let $F(\theta_t)$ be the value corresponding to parameter $\theta_t$, and $F_{\min}$ and $\bar{F}$ be the minimum and mean values of $F$. To satisfy $\mathcal{L}^{\prime\prime}(\theta_{t}) > 1$ while preserving relative parameter importance, we introduce a scaling hyperparameter $\alpha > 0$ to regulate the distribution of $\beta_t$:
\begin{equation} \label{eq:fisher_scaling}
    \mathcal{L}^{\prime\prime}(\theta_t) \approx \alpha \frac{F(\theta_t) - F_{\min}}{\bar{F} - F_{\min}} + 1.
\end{equation}
This ensures the approximated Hessian lower bound is 1. Substituting Eq. \eqref{eq:fisher_scaling} into Eq. \eqref{eq:optimal_beta} yields the computable coefficient:
\begin{equation} \label{eq:final_beta}
    \beta_t = \frac{(\bar{F} - F_{\min})\left[(\theta_p + \theta_{t-1}^* - 2\theta_t) - \mathcal{L}^{\prime}(\theta_t)\right]}{(\theta_p + \theta_{t-1}^* - 2\theta_t)\left[\alpha F(\theta_t) - \alpha F_{\min} + 2\bar{F} - 2F_{\min}\right]}.
\end{equation}
In practice, we clip $\beta_t$ to $[0.001, 0.499]$ for numerical stability.

\subsection{Robust Initialization via Historical Knowledge}
\label{sec:initialization}

As established in Section \ref{sec:theory}, $\theta_p$ serves as the center of the prior distribution $P_{\text{init}}$. Instead of standard random initialization, we utilize the average of previously learned adapters to improve stability. For a new task $t > 1$, we define $\theta_p$ as the arithmetic mean of the optimized parameters from the preceding $t-1$ tasks:
\begin{equation} \label{eq:init_assignment}
    \theta_p = \theta_{avg}^{t-1} = \frac{1}{t-1}\sum_{i=1}^{t-1}\theta_i,
\end{equation}
where $\theta_i$ denotes the optimized parameters for task $i$. For the first task ($t=1$), standard random initialization is applied.

Directly computing Eq. \eqref{eq:init_assignment} requires storing all historical parameters, which causes high memory overhead. To avoid this, we use a recursive update. After training task $t$, we update the running average as:
\begin{equation} \label{eq:running_avg}
    \theta_{avg}^{t} = \frac{t-1}{t}\theta_{avg}^{t-1} + \frac{1}{t}\theta_{t}.
\end{equation}
By maintaining only the running average $\theta_{avg}$, this method constructs a robust initialization $\theta_p$ containing common historical patterns without requiring additional storage.

\subsection{Overview of Our Method}

The DAF execution procedure for each task $t$ consists of four key steps. First, we initialize the task adapter $\theta_t$ using the historical prior $\theta_p$ (Eq. \eqref{eq:init_assignment}) and train it on the current dataset $\mathcal{D}_t$. Second, we compute the optimal fusion coefficient $\beta_t$ via Eq. \eqref{eq:final_beta} using the gradient and Fisher information derived from $\theta_t$. Third, we fuse $\theta_t$, $\theta_p$, and the previous global state $\theta_{t-1}^*$ to update the global adapter $\theta_{t}^*$ (Eq. \eqref{eq:fusion}). Finally, we update the running average $\theta_{avg}^{t}$ (Eq. \eqref{eq:running_avg}) and discard all temporary parameters to ensure memory efficiency. Detailed algorithmic steps are provided in Appendix~\ref{algorithm}.

\begin{table*}[t]
\centering
\caption{
    Average ($\bar{\mathcal{A}}$) and final task ($\mathcal{A}_T$) Top-1 accuracy comparison on four challenging class-incremental learning benchmarks. 
    All methods use ViT-B/16-IN21K as the backbone and are evaluated in the exemplar-free setting.
    The specific task configurations (e.g., number of base classes and incremental steps) are detailed in the experimental setup section.
    The best performance in each column is highlighted in \textbf{bold}.
    $\dagger$ Results of the single global adapter.
}
\label{tab:main_results}
\begin{tabular*}{\textwidth}{@{\extracolsep{\fill}}l|cc|cc|cc|cc}
    \toprule
    \multicolumn{1}{c|}{\multirow{2}{*}{\textbf{Method}}} & 
    \multicolumn{2}{c|}{\textbf{CIFAR-100}} & 
    \multicolumn{2}{c|}{\textbf{ImageNet-R}} &
    \multicolumn{2}{c|}{\textbf{ImageNet-A}} & 
    \multicolumn{2}{c}{\textbf{ObjectNet}} \\
    \cmidrule(lr){2-3} \cmidrule(lr){4-5} \cmidrule(lr){6-7} \cmidrule(lr){8-9}
    & $\bar{\mathcal{A}}$ (\%) & $\mathcal{A}_T$ (\%) & 
      $\bar{\mathcal{A}}$ (\%) & $\mathcal{A}_T$ (\%) &
      $\bar{\mathcal{A}}$ (\%) & $\mathcal{A}_T$ (\%) &
      $\bar{\mathcal{A}}$ (\%) & $\mathcal{A}_T$ (\%) \\
    \midrule
    Finetune         & 38.90 & 20.17 & 32.32 & 22.78 & 24.28 & 14.51  & 19.14 & 8.73 \\
    \midrule
    L2P~\cite{wang2022learning}            & 85.94 & 79.93 & 75.46 & 69.77 & 49.39 & 41.71 & 63.78 & 52.19 \\
    DualPrompt~\cite{wang2022dualprompt}   & 87.87 & 81.15 & 73.10 & 67.18 & 53.71 & 41.67 & 59.27 & 49.33 \\
    CODA-Prompt~\cite{smith2023coda}     & 89.11 & 81.96 & 77.97 & 72.27 & 53.54 & 42.73 & 66.07 & 53.29 \\
    SLCA~\cite{zhang2023slca}              & 92.49 & 88.55 & 81.17 & 77.00 & 68.66 & 58.74 & 72.55 & 61.30 \\ 
    RanPAC~\cite{mcdonnell2023ranpac}    & 94.00 & 90.62 & 82.98 & 77.94 & 69.32 & 61.82 & 72.76 & 62.02 \\
    SSIAT~\cite{tan2024semantically}             & 93.52 & 90.07 & 83.20 & 78.85 & 70.83 & 62.23 & 73.65 & 62.45 \\
    SimpleCIL~\cite{zhou2024revisiting}    & 87.57 & 81.26 & 61.26 & 54.55 & 59.77 & 48.91 & 65.45 & 53.59 \\
    APER + Adapter~\cite{zhou2024revisiting} & 90.65 & 85.15 & 75.82 & 67.95 & 60.47 & 49.37 & 67.18 & 55.24 \\
    EASE~\cite{zhou2024expandable}         & 91.51 & 85.80 & 81.74 & 76.17 & 65.34 & 55.04 & 70.84 & 57.86 \\
    TUNA$\dagger$~\cite{wang2025integrating} & 92.68 & 86.19 & 81.24 & 72.92 & 56.14 & 28.64 & 72.61 & 58.52 \\
    MOS~\cite{sun2025mos}                  & 93.30 & 89.25 & 82.96 & 77.93 & 67.08 & 56.22 & 74.69 & 63.62 \\
    \midrule
    \textbf{\method (Ours)}                   & \bf{94.58} & \bf{91.15} & \bf{84.01} & \bf{79.63} & \bf{72.06} & \bf{62.54} & \bf{76.11} & \bf{66.37} \\
    \bottomrule
\end{tabular*}
\end{table*}
\begin{figure*}[t]
    \centering
    \begin{subfigure}[b]{0.24\textwidth}
        \centering
        \includegraphics[width=\textwidth]{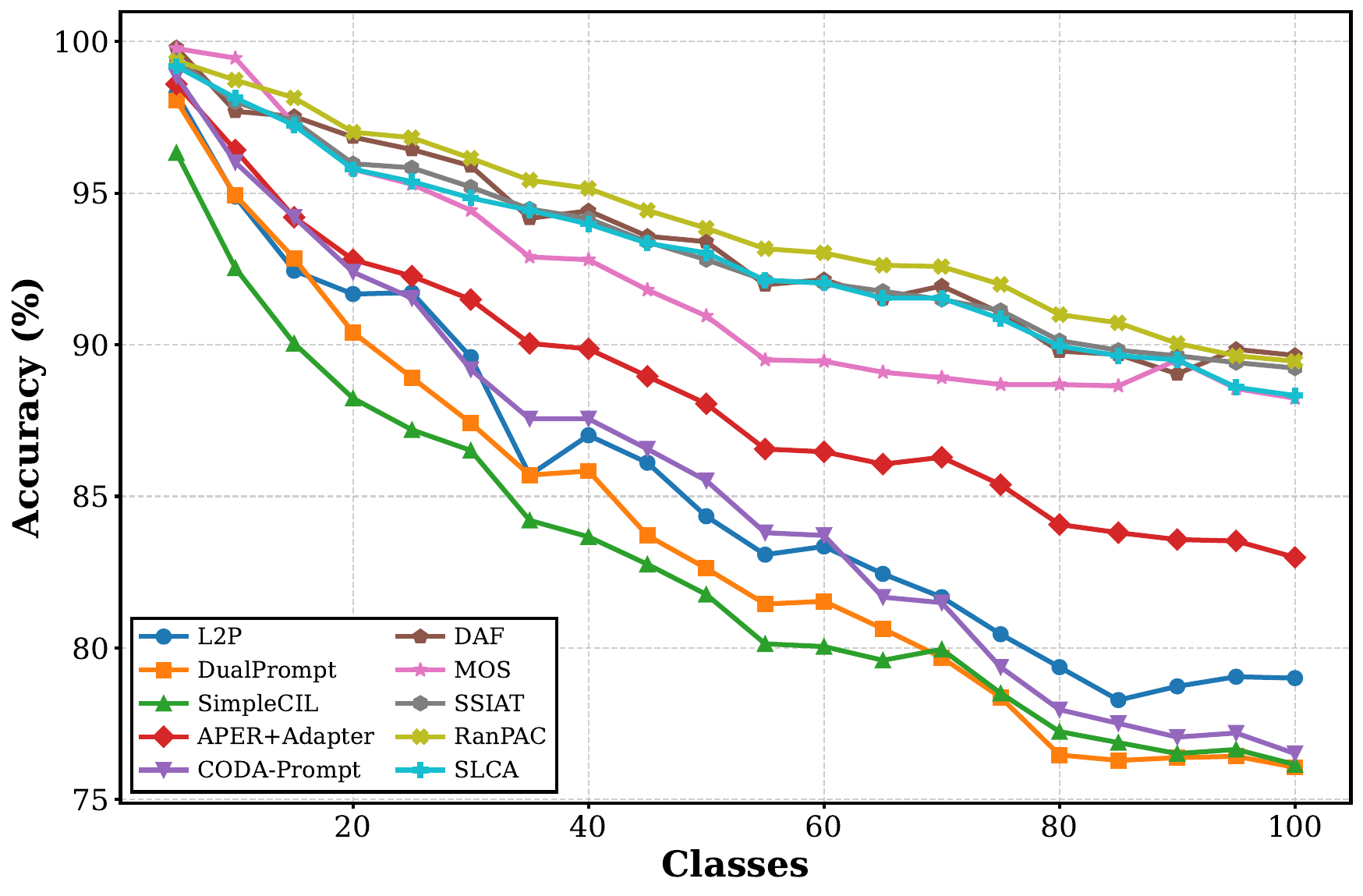} 
        \caption{CIFAR-100}
        \label{subfig:curve_cifar100}
    \end{subfigure}
    \hfill
    \begin{subfigure}[b]{0.24\textwidth}
        \centering
        \includegraphics[width=\textwidth]{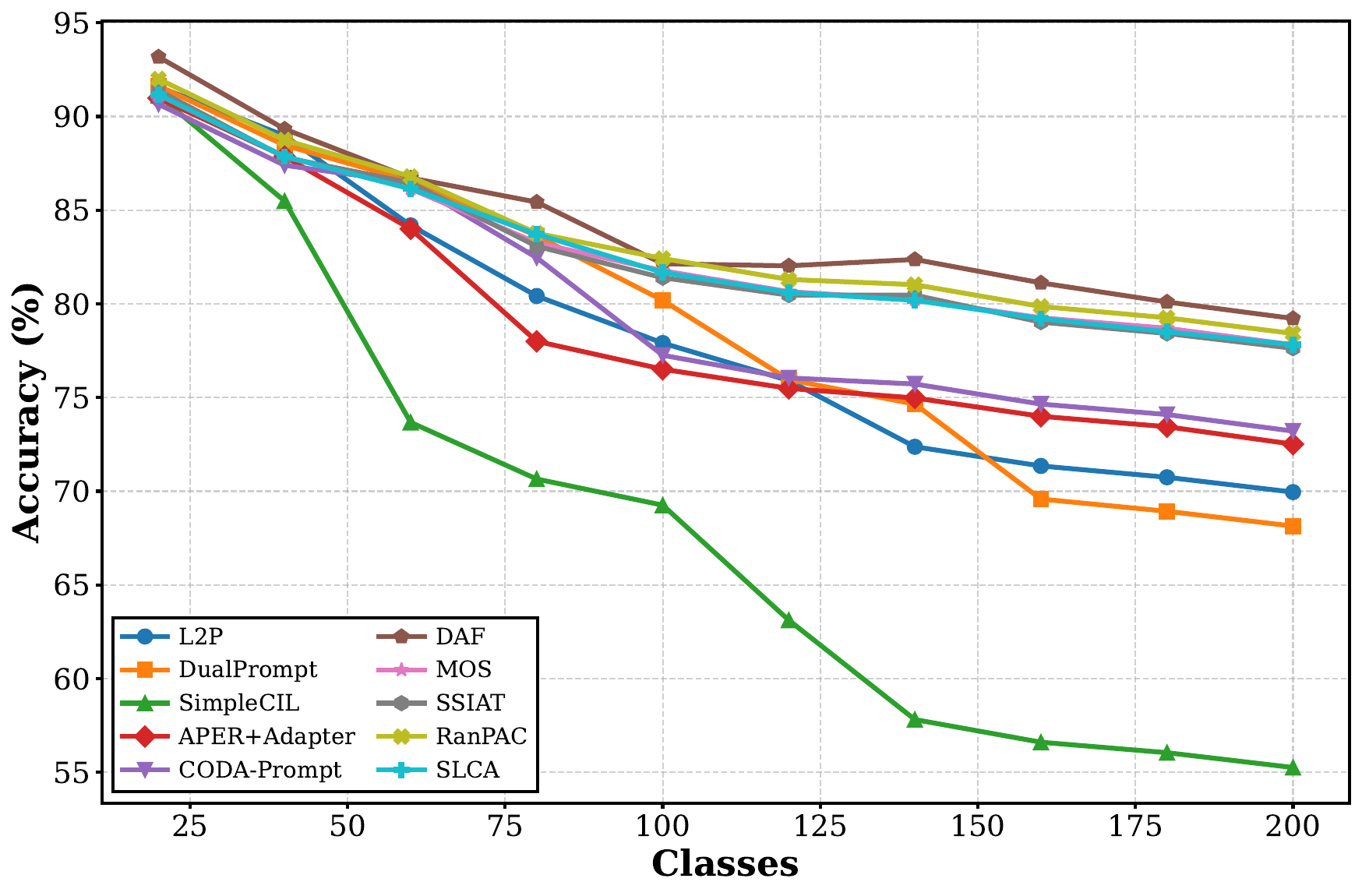}
        \caption{ImageNet-R}
        \label{subfig:curve_imr}
    \end{subfigure}
    \hfill
    \begin{subfigure}[b]{0.24\textwidth}
        \centering
        \includegraphics[width=\textwidth]{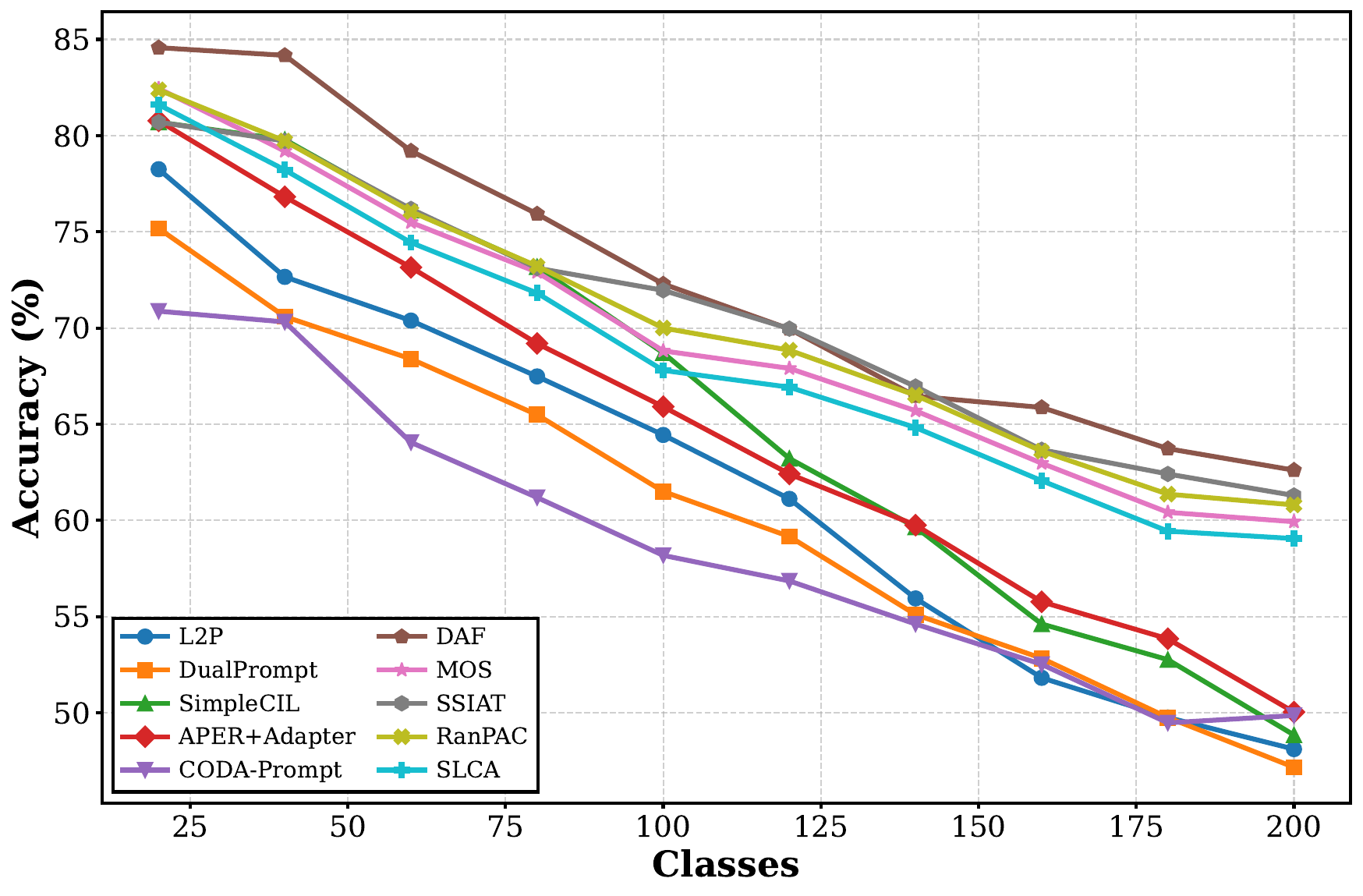}
        \caption{ImageNet-A}
        \label{subfig:curve_ima}
    \end{subfigure}
    \hfill
    \begin{subfigure}[b]{0.24\textwidth}
        \centering
        \includegraphics[width=\textwidth]{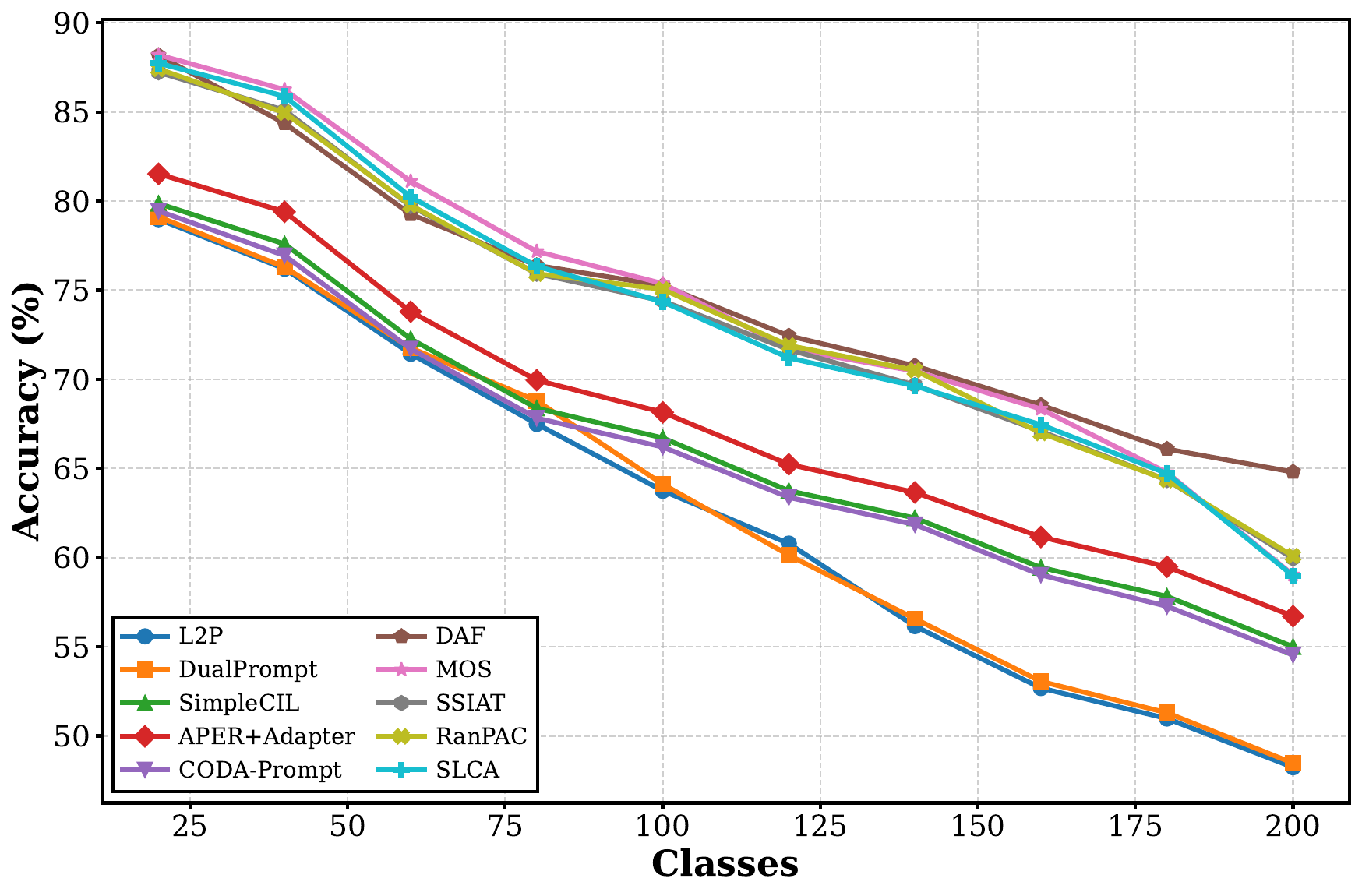}
        \caption{ObjectNet}
        \label{subfig:curve_obj}
    \end{subfigure}

    \caption{Incremental accuracy trends on four benchmarks using the ViT-B/16-IN1K backbone. 
    DAF (red curve) is compared against state-of-the-art exemplar-free methods. 
    Across all datasets, DAF consistently demonstrates superior performance. 
    The curves highlight our method's robust ability to mitigate catastrophic forgetting and maintain high accuracy as the number of tasks increases.}
    \label{fig:main_results_curves}
\end{figure*}
\section{Experiments}
\label{sec:experiments}

\subsection{Implementation Details}

\noindent\textbf{Datasets and Task Protocols.}
We follow the standard evaluation protocol established by~\cite{zhou2024revisiting} to ensure a fair comparison. We conduct experiments on four widely used CIL benchmarks: CIFAR-100~\cite{krizhevsky2009learning}, ImageNet-R~\cite{hendrycks2021many}, ImageNet-A~\cite{hendrycks2021natural}, and ObjectNet~\cite{barbu2019objectnet}. These datasets cover both standard classification scenarios and challenging settings with significant domain shifts. Regarding the task construction, CIFAR-100 is divided into 20 tasks with 5 classes per task. The remaining three datasets, each containing 200 classes, are divided into 10 tasks with 20 classes per task. Consistent with~\cite{rebuffi2017icarl}, we shuffle the class order using a fixed random seed of 1993 before splitting the tasks.

\noindent\textbf{Training Details.}
We implement all experiments using PyTorch~\cite{paszke2019pytorch} based on the PILOT codebase~\cite{sun2023pilot}. The training is conducted on NVIDIA GeForce RTX 2080 Ti GPUs. For fair comparison, we conduct experiments using two distinct backbones: ViT-B/16-IN21K (pre-trained on ImageNet-21K) and ViT-B/16-IN1K (pre-trained on ImageNet-1K). For each task, we train the models for 20 epochs with a batch size of 48. We use an SGD optimizer with momentum, setting the initial learning rate to 0.01 and applying a cosine annealing schedule for learning rate decay. For the adapter configuration, the projection dimension is set to 16 for all adapter-based methods. Regarding the hyperparameters of our method, the scaling factor $\alpha$ in Eq.~\ref{eq:fisher_scaling} is set to 1.25.
\begin{figure*}[t]
    \centering
    \begin{subfigure}[b]{0.32\textwidth}
        \centering
        \includegraphics[width=\textwidth]{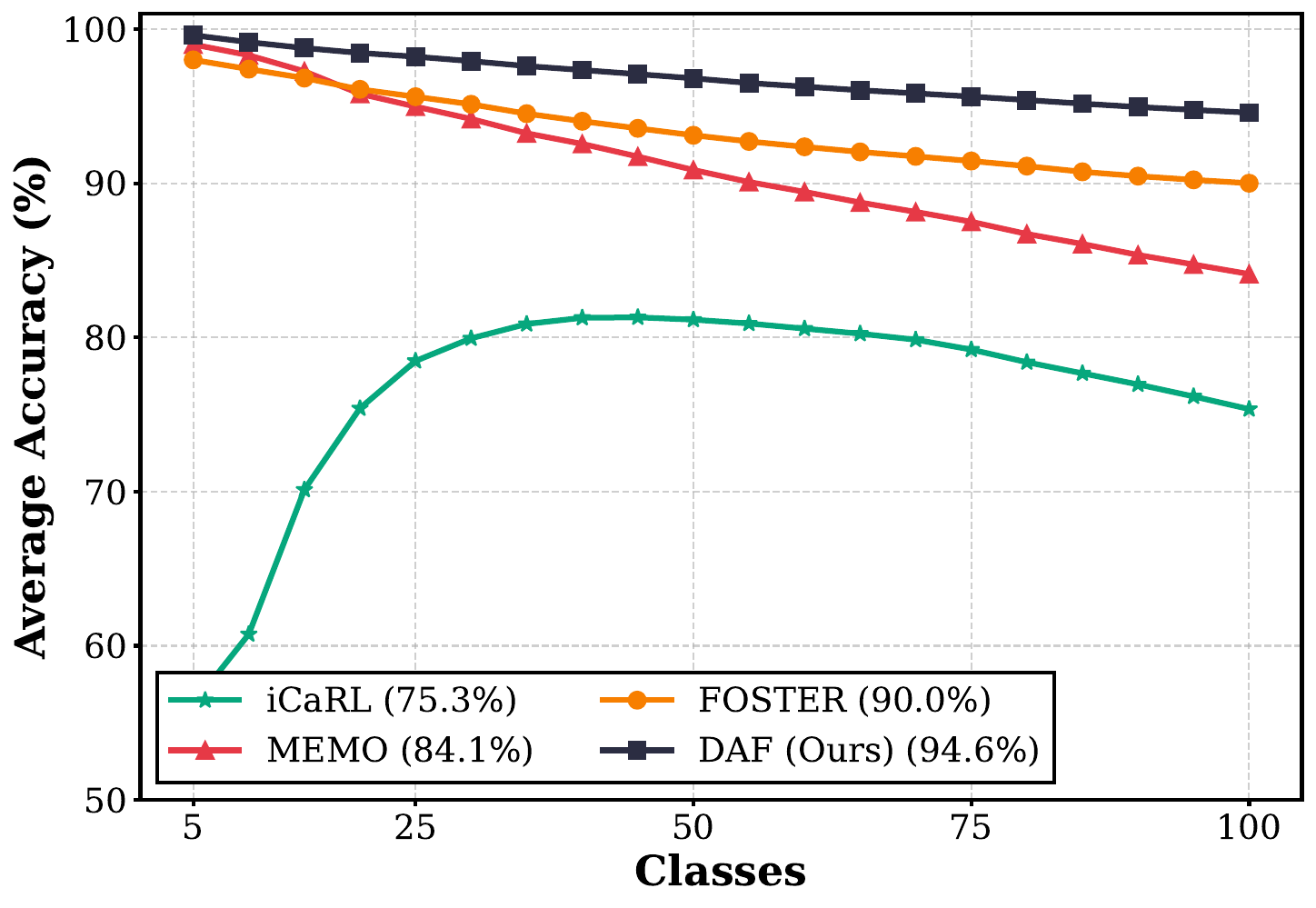} 
        \caption{CIFAR-100}
        \label{subfig:rehearsal_cifar}
    \end{subfigure}
    \hfill
    \begin{subfigure}[b]{0.32\textwidth}
        \centering
        \includegraphics[width=\textwidth]{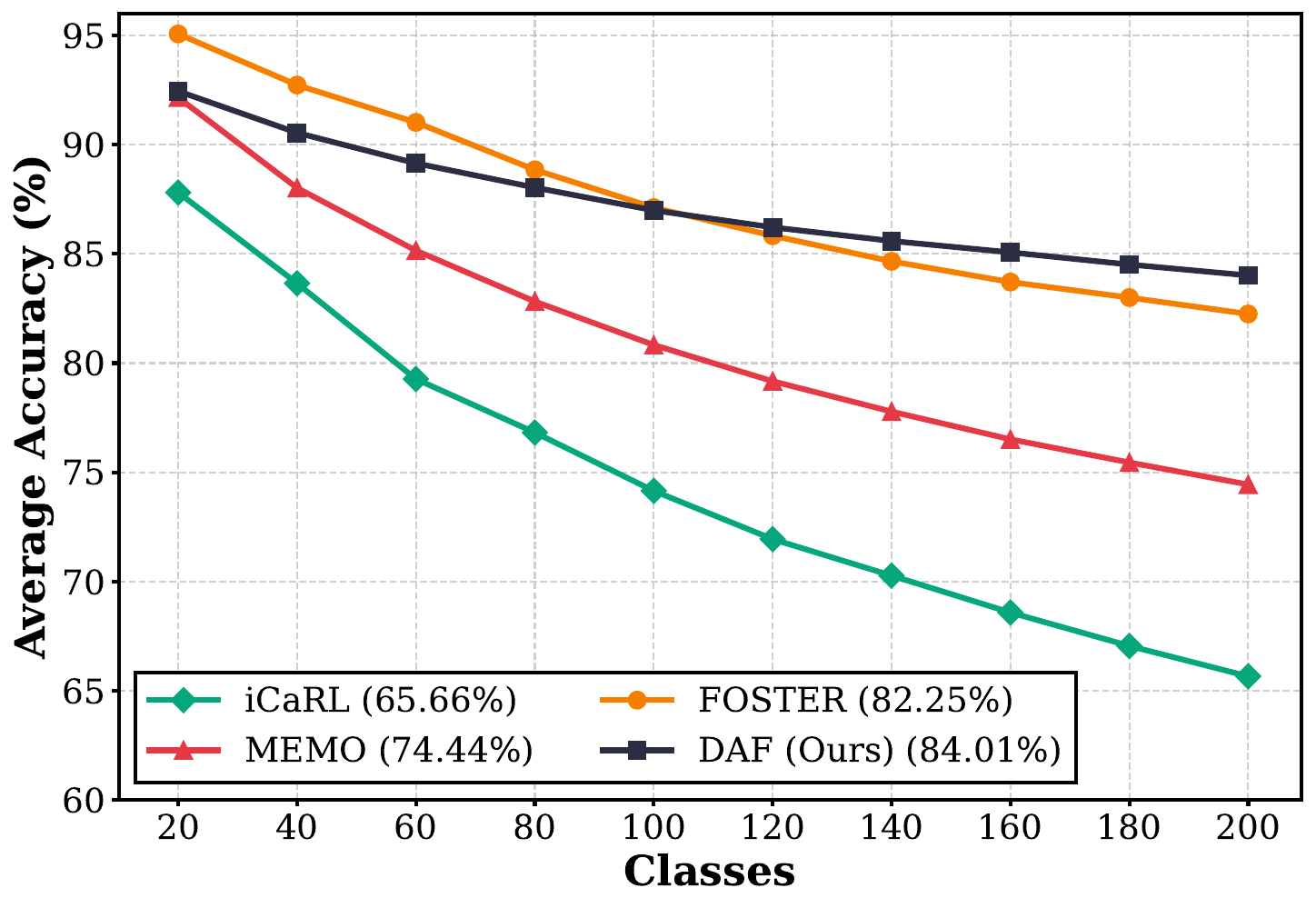} 
        \caption{ImageNet-R}
        \label{subfig:rehearsal_imr}
    \end{subfigure}
    \hfill
    \begin{subfigure}[b]{0.32\textwidth}
        \centering
        \includegraphics[width=\textwidth]{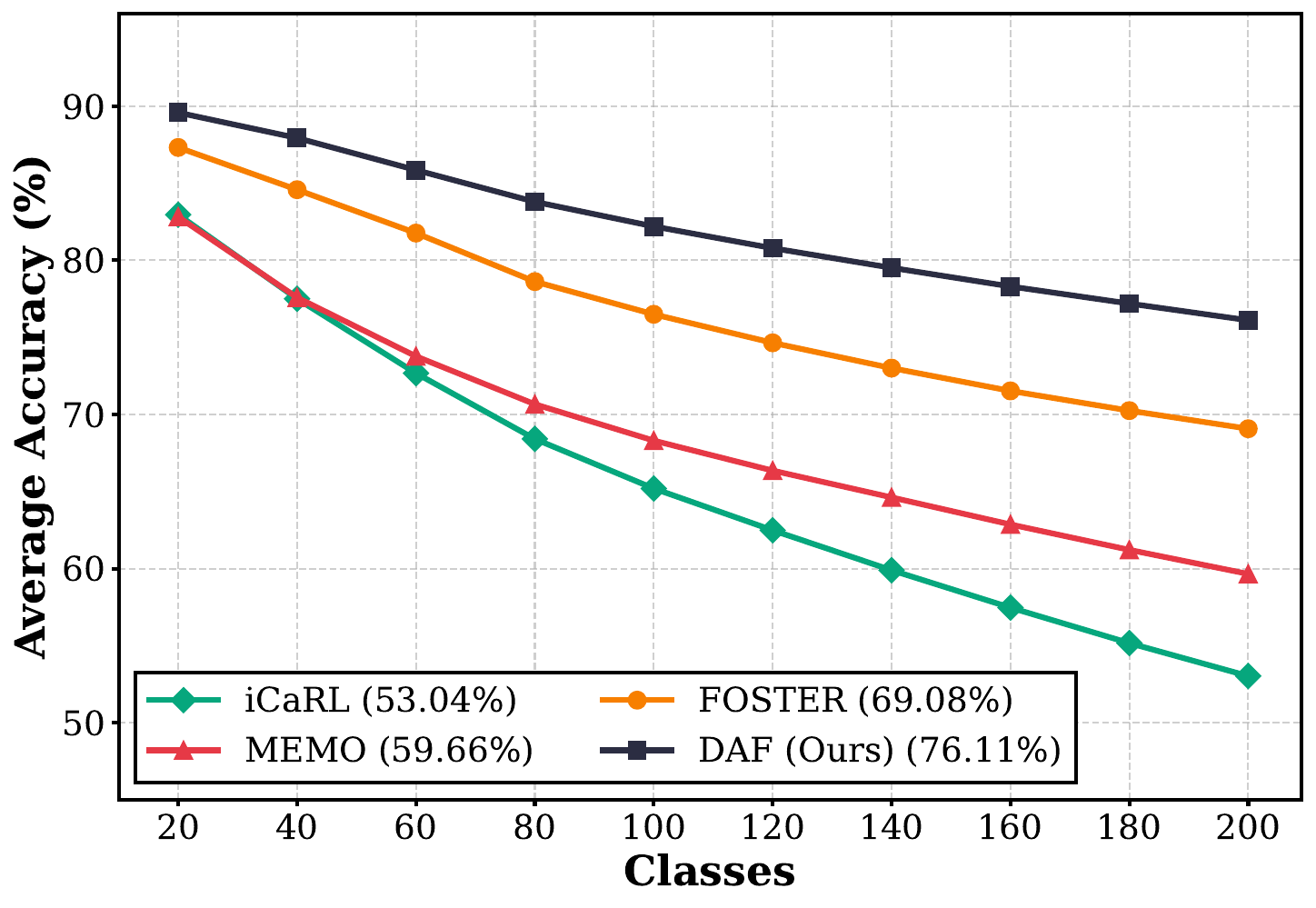} 
        \caption{ObjectNet}
        \label{subfig:rehearsal_obj}
    \end{subfigure}
    \caption{Comparison with rehearsal-based methods. Rehearsal-based methods use 20 exemplars per class. Using ViT-B/16-IN21K as the backbone, DAF significantly outperforms strong replay baselines (iCaRL, MEMO, FOSTER) across CIFAR-100, ImageNet-R, and ObjectNet. Notably, DAF achieves this without storing past exemplars, validating the effectiveness of our parameter fusion mechanism over traditional sample replay.}
    \label{fig:rehearsal_comparison}
\end{figure*}

\noindent\textbf{Comparison Methods.}
We compare DAF against state-of-the-art CIL methods. We mainly compare with PTM-based methods, including L2P~\cite{wang2022learning}, DualPrompt~\cite{wang2022dualprompt}, CODA-Prompt~\cite{smith2023coda}, RanPAC~\cite{mcdonnell2023ranpac}, SimpleCIL~\cite{zhou2024revisiting}, APER~\cite{zhou2024revisiting}, SLCA~\cite{zhang2023slca}, SSIAT~\cite{tan2024semantically}, EASE~\cite{zhou2024expandable}, MOS~\cite{sun2025mos}, and TUNA~\cite{wang2025integrating}. We also include rehearsal-based methods adapted to the PTM setting, such as FOSTER~\cite{wang2022foster}, MEMO~\cite{zhou2022model}, and iCaRL~\cite{rebuffi2017icarl}. Additionally, we report the performance of the 'Finetune' baseline, which sequentially fine-tunes the entire PTM on each task. This baseline serves as a lower bound to indicate the degree of catastrophic forgetting.

\noindent\textbf{Evaluation Metrics.}
We follow the standard CIL evaluation protocol defined in~\cite{rebuffi2017icarl}. Performance is measured using two key metrics. 
The \textbf{Average Accuracy} ($\bar{\mathcal{A}}$) measures the overall performance across the entire learning trajectory: $\bar{\mathcal{A}} = \frac{1}{T} \sum_{t=1}^{T} \left( \frac{1}{t} \sum_{j=1}^{t} A_{t,j} \right)$, where $A_{t,j}$ denotes the accuracy on task $j$ after training on task $t$.
The \textbf{Final Average Accuracy} ($\mathcal{A}_T$) evaluates the mean performance on all tasks after the final session: $\mathcal{A}_T = \frac{1}{T} \sum_{j=1}^{T} A_{T,j}$.
Additionally, we introduce two metrics to analyze the trade-off between stability and plasticity. 
\textbf{Stability} ($S$) measures the retention of knowledge on previous tasks after the final training stage: $S = \frac{1}{T-1}\sum_{j=1}^{T-1} A_{T,j}$. 
\textbf{Plasticity} ($P$) evaluates the learning capability on new tasks by averaging the immediate accuracy of each task right after it is learned: $P = \frac{1}{T}\sum_{j=1}^{T} A_{j,j}$. 
Higher values for all metrics indicate better performance.
\subsection{Main Results}
\label{sec:main_results}

\paragraph{Comparison with SOTA Exemplar-Free Methods.} 
We evaluate DAF against SOTA exemplar-free methods, with detailed results presented in Table~\ref{tab:main_results}. 
Across all four benchmarks, DAF achieves the highest Average Accuracy ($\bar{\mathcal{A}}$) and Final Accuracy ($\mathcal{A}_T$). 
For instance, on the challenging ObjectNet benchmark, DAF achieves a final accuracy of 66.37\%, outperforming the leading retrieval-based method, MOS, by 2.75\%. 
This result supports our central claim that constructing a theoretically grounded global adapter is a more effective strategy for mitigating interference than relying on adapter retrieval mechanisms. 
While retrieval-based methods like MOS are effective, their performance heavily depends on the accuracy of the selector, which can be unstable under domain shifts. 
DAF avoids this dependency by rigorously fusing historical knowledge into a single global adapter using PAC-Bayes bounds.

To further evaluate robustness with a weaker backbone, we conduct experiments using ViT-B/16-IN1K. Figure~\ref{fig:main_results_curves} illustrates the evolution of Average Accuracy ($\bar{\mathcal{A}}$) as the number of observed classes increases. Consistent with the IN21K results, DAF maintains a clear advantage throughout the incremental sequence across all benchmarks. Notably, on the difficult ObjectNet dataset, DAF exhibits a significantly slower decay in $\bar{\mathcal{A}}$ compared to the runner-up method MOS, underscoring its superior stability. These results confirm that the effectiveness of our dynamical fusion mechanism remains agnostic to the strength of the pre-trained backbone.

\paragraph{Comparison with Strong Rehearsal-Based Methods.}
We also compare DAF against strong rehearsal-based methods, despite our method being exemplar-free. As shown in Figure~\ref{fig:rehearsal_comparison}, DAF surpasses these methods on CIFAR-100, ImageNet-R, and ObjectNet. On the 10-task ObjectNet sequence, DAF achieves a final average accuracy of 76.11\%, establishing a significant lead of 7.03\% over the state-of-the-art replay method, FOSTER. Notably, the performance gap widens as more tasks are learned. This suggests that our dynamical fusion method, which rigorously integrates historical priors, is a more scalable solution for lifelong learning than relying on the overhead of storing and replaying limited exemplars.
\begin{figure*}[t!]
    \centering
    
    \begin{subfigure}[b]{0.32\textwidth}
        \centering
        \includegraphics[width=\linewidth]{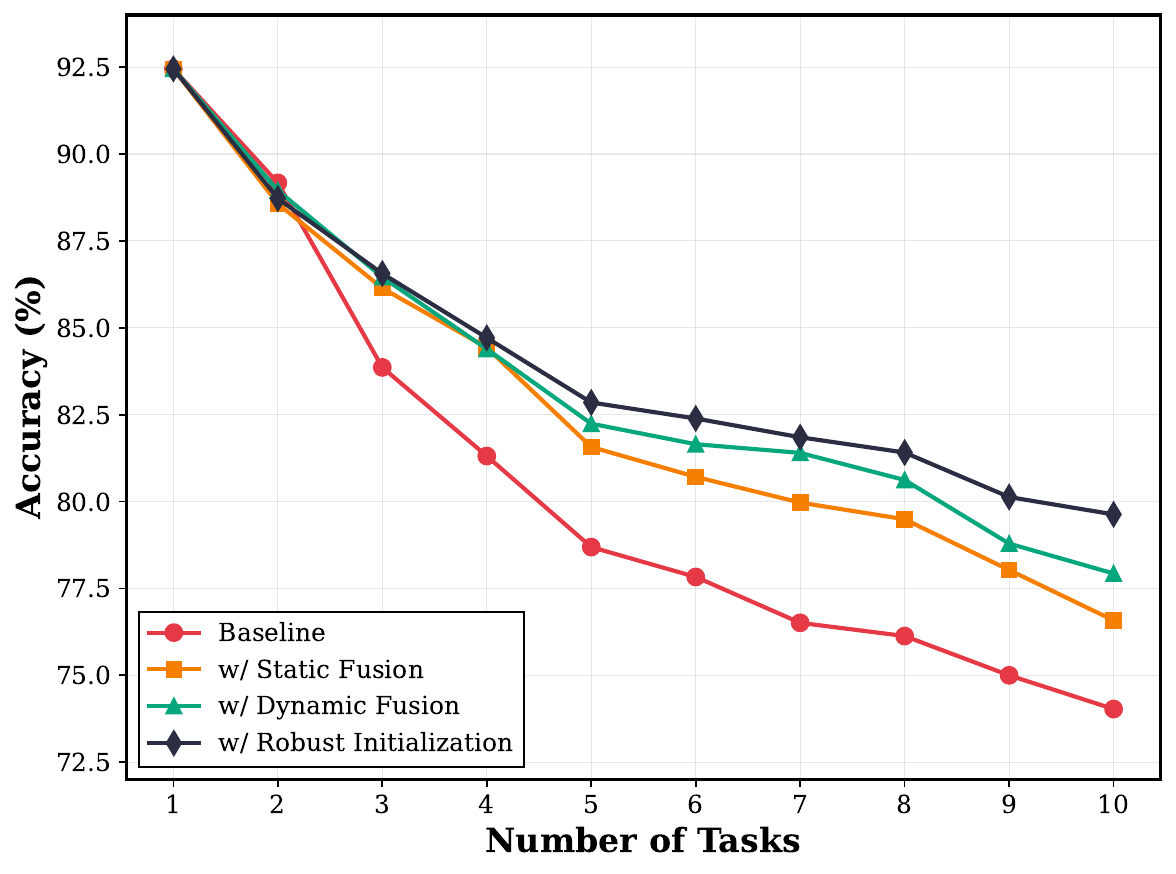}
        \caption{Ablation of DAF Components}
        \label{fig:ablation_sub}
    \end{subfigure}
    \hfill 
    \begin{subfigure}[b]{0.32\textwidth}
        \centering
        \includegraphics[width=\linewidth]{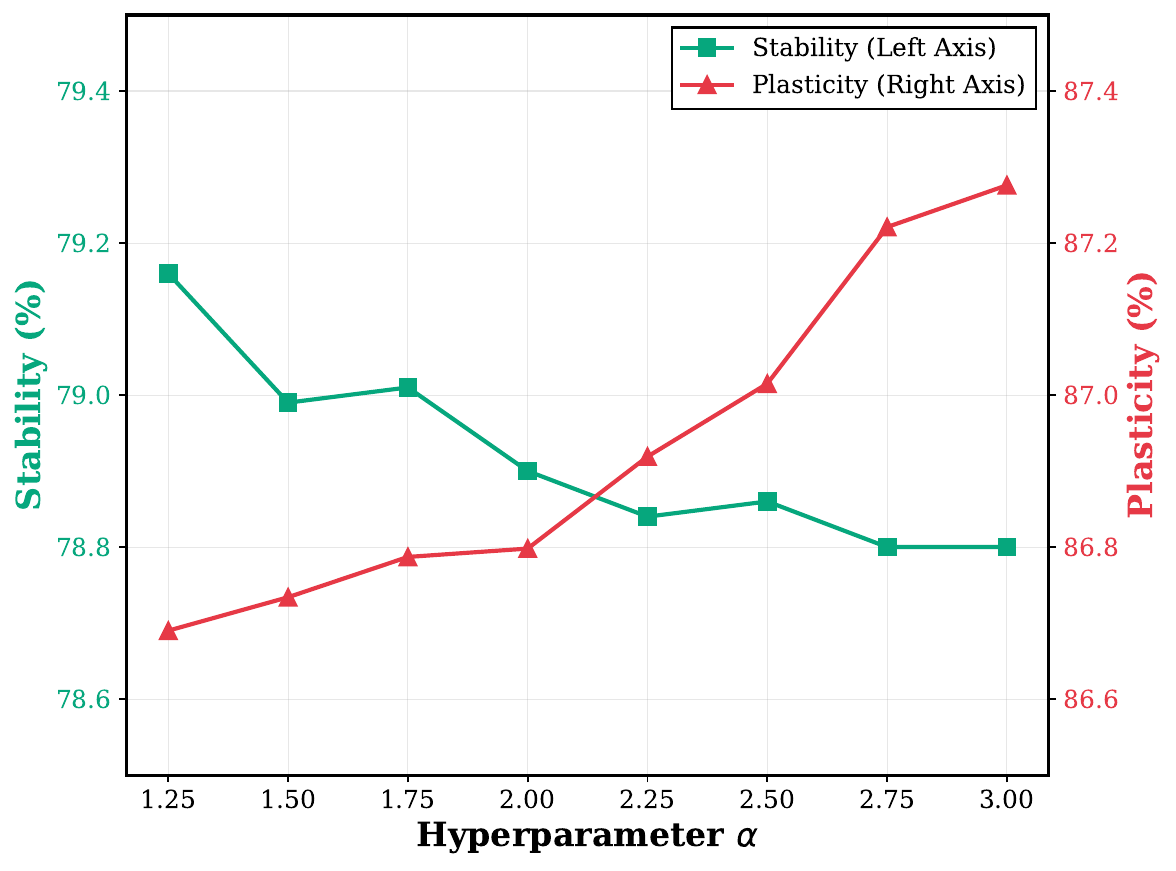}
        \caption{Sensitivity of Hyperparameter $\alpha$}
        \label{fig:alpha_sub}
    \end{subfigure}
    \hfill
    \begin{subfigure}[b]{0.32\textwidth}
        \centering
        \includegraphics[width=\linewidth]{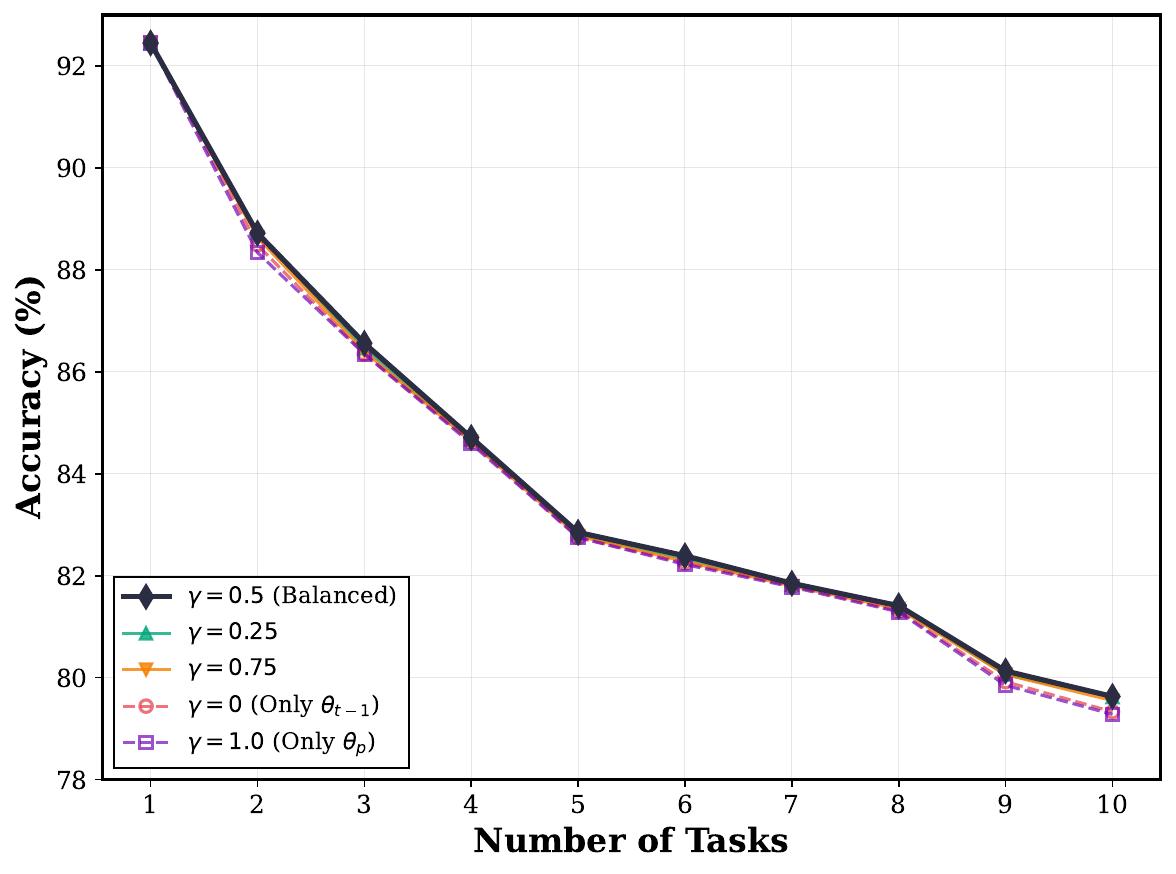}
        \caption{Impact of Prior Weight $\gamma$}
        \label{fig:gamma_sub}
    \end{subfigure}

    \caption{
        Ablation and sensitivity analysis of DAF on ImageNet-R using the ViT-B/16-IN21K backbone.
        \textbf{(a)} Incremental ablation study demonstrating the cumulative gains from Static Fusion, Dynamical Fusion, and Robust Initialization.
        \textbf{(b)} Dual-axis analysis of hyperparameter $\alpha$, illustrating the explicit trade-off between Stability (Green) and Plasticity (Red).
        \textbf{(c)} Evaluation of the balancing weight $\gamma$, confirming that the theoretically derived equal weighting ($\gamma=0.5$) of historical priors yields optimal performance.
    }
    \label{fig:analysis_row}
\end{figure*}

\subsection{Ablation Study}
\label{sec:ablation}

To dissect the contribution of each key component in DAF, we conduct a progressive ablation study on ImageNet-R (Figure~\ref{fig:ablation_sub}).
Our analysis begins with a \textbf{Baseline} that trains task-specific adapters with random initialization and uses the final task's adapter for inference.
First, we introduce \textbf{Static Fusion}, which aggregates initialization, task-specific, and previous global parameters with fixed weights ($\beta=1/3$). This yields an immediate performance uplift, highlighting the benefit of retaining historical information to mitigate forgetting.
Building upon this, we incorporate \textbf{Dynamical Fusion}. By replacing rigid weights with adaptive coefficients $\beta_t$ derived from our PAC-Bayes framework, this mechanism significantly enhances performance, effectively balancing the trade-off between empirical risk and stability.
Finally, the full DAF method integrates \textbf{Robust Initialization}. Initializing the new adapter with the running average of past parameters acts as a strong informative prior. This prevents starting optimization from scratch and guides the model toward a solution compatible with historical patterns. Each component thus plays a distinct and synergistic role, culminating in the robust performance of DAF.

\subsection{Discussion}

\paragraph{Hyperparameter Sensitivity.}

Figure~\ref{fig:alpha_sub} analyzes the impact of the constraint strength $\alpha$. As $\alpha$ increases from 1.25 to 3.0, the relaxed historical constraint improves Plasticity but slightly reduces Stability. Despite this trade-off, overall performance remains highly robust (see Appendix~\ref{app:sensitivity}), indicating that DAF is not sensitive to precise tuning. We set $\alpha=1.25$ to prioritize stability.

\begin{table}[t]
\centering
\caption{
    Impact of Initialization on ImageNet-R using the ViT-B/16-IN21K backbone.
    We compare Random, Previous Task, and our Robust Initialization under two inference paradigms: utilizing the Last Task Adapter vs. our Global Fused Adapter (DAF).
    The best performance in each column is highlighted in \textbf{bold}. Our Robust Initialization consistently yields the best results across both paradigms.
}
\label{tab:init_discussion}
\setlength{\tabcolsep}{8pt}
\begin{tabular}{l|cc}
\toprule
\textbf{Initialization Method} & $\mathcal{A}_T$ (\%) & $\bar{\mathcal{A}}$ (\%) \\
\midrule
\multicolumn{3}{l}{\textit{Paradigm 1: Last Task Adapter}} \\
\midrule
\quad Random Initialization            & 80.50 & 74.03 \\
\quad Previous Task ($\theta_{t-1}$)   & 82.65 & 77.10 \\
\quad Robust Initialization (Ours)              & 83.57 & 77.90 \\
\midrule
\multicolumn{3}{l}{\textit{Paradigm 2: Global Fused Adapter (DAF)}} \\
\midrule
\quad Random Initialization            & 78.50 & 71.52 \\
\quad Previous Task ($\theta_{t-1}$)   & 82.98 & 78.58 \\
\quad Robust Initialization (Ours)              & \textbf{84.01} & \textbf{79.63} \\
\bottomrule
\end{tabular}
\end{table}
\paragraph{Balancing Prior Knowledge.}
Our derived fusion objective inherently assigns equal weight to the recursive historical average $\theta_p$ and the previous global adapter $\theta_{t-1}^*$. To validate this design, we introduce a balancing coefficient $\gamma$ to re-weight these terms according to the following formulation: 
$\theta_{t}^{*} = 2\gamma\beta_{t}\theta_{p} + 2(1-\gamma)\beta_{t}\theta_{t-1}^{*} + \left(1 - 2\beta_{t}\right)\theta_{t}$.
Figure~\ref{fig:gamma_sub} presents the performance on ImageNet-R across different $\gamma$ values. 
We observe that the balanced setting ($\gamma=0.5$), which recovers our original theoretical derivation, achieves the optimal accuracy of 79.63\%. 
Shifting towards extremes leads to degradation: setting $\gamma=0$ (ignoring the robust average $\theta_p$) or $\gamma=1$ (ignoring the specific previous state $\theta_{t-1}^*$) results in lower performance (79.33\% and 79.28\%, respectively). 
This finding suggests that $\theta_p$ and $\theta_{t-1}^*$ capture complementary aspects of past knowledge, representing general historical patterns and specific recent states respectively, and that effectively combining them is crucial for maximizing performance.

\paragraph{Effectiveness of Robust Initialization.}
We further investigate the role of initialization strategies across two inference paradigms: using the \textit{Last Task Adapter} and our \textit{Global Fused Adapter}. Table~\ref{tab:init_discussion} details the comparison results.
A critical observation is that DAF with random initialization yields an Average Accuracy of 71.52\%, which is notably inferior to the Last Task Adapter baseline (74.03\%). 
This performance drop occurs because integrating randomly initialized parameters into the fusion process introduces significant noise, disrupting the stability of the global model.
However, when equipped with our Robust Initialization, DAF effectively mitigates this issue, achieving the state-of-the-art performance of 79.63\%.
This contrast highlights that while dynamical fusion is powerful, its success heavily relies on a high-quality initialization ($\theta_p$) to ensure that the merged parameters are constructive rather than destructive.

\section{Conclusion}
To address the high inference latency and potential retrieval errors inherent in current Pre-Trained Model-based CIL methods, we propose Dynamical Adapter Fusion, a novel framework designed to construct a single robust global adapter.
DAF formulates the parameter fusion process as a constrained optimization problem derived from PAC-Bayes theorem.
By employing the Method of Lagrange Multipliers, we derive optimal weights to dynamically fuse task-specific and global parameters, ensuring an optimal balance between stability and plasticity without storing historical adapters.
Furthermore, we propose a Robust Initialization strategy that leverages recursive historical knowledge to effectively capture global patterns.
Extensive experiments on major CIL benchmarks demonstrate that DAF achieves superior accuracy over SOTA methods while achieving high training and inference efficiency. We discuss the limitations and future work in Appendix~\ref{app:limitations}

\section{Impact Statements}

Our proposed DAF framework presents a significant advancement in Class-Incremental Learning by enabling efficient, lifelong knowledge accumulation without the burden of catastrophic forgetting.
By eliminating the need for expanding architectures or inference-time retrieval, DAF drastically reduces computational and memory overhead. 
This efficiency is particularly relevant for real-world applications on resource-constrained edge devices, such as autonomous robots and mobile assistants, where low latency is critical.
Ethically, our method adheres to a strict exemplar-free protocol. 
By avoiding the storage of historical samples, DAF mitigates privacy risks associated with sensitive data retention, promoting the responsible deployment of AI in privacy-critical domains like healthcare.
Furthermore, the high training and inference throughput of our method contributes to sustainable AI development by minimizing the energy footprint required for continuous model evolution.
\nocite{langley00}

\bibliography{example_paper}
\bibliographystyle{icml2026}

\newpage
\appendix
\onecolumn
\section{Limitations and Future Work}
\label{app:limitations}

Despite the strong performance and efficiency of DAF, we acknowledge certain limitations that offer avenues for future research.

First, similar to most standard Class-Incremental Learning (CIL) methods, our current framework relies on the availability of clear task boundaries during the training phase to perform the dynamical fusion.
While this assumption holds for many structured learning scenarios, extending DAF to a fully Task-Free setting, where task boundaries are ambiguous or unknown, remains a challenging but promising direction.

Second, as a PEFT method, the upper bound of DAF's performance is naturally influenced by the representational power of the frozen pre-trained backbone.
While we have verified efficacy on Vision Transformers, exploring how our dynamical fusion strategy generalizes to other architectures, such as larger Foundation Models, warrants further investigation.

Third, regarding our theoretical formulation, the derivation of the optimal fusion coefficient $\beta_t$ relies on a second-order Taylor expansion of the loss function and approximates the Hessian diagonal using the Fisher Information Matrix.
The claim of theoretical optimality therefore rests on the accuracy of these approximations.
In scenarios with highly irregular loss landscapes where the local quadratic assumption may not hold, the derived coefficients could potentially deviate from the ideal solution, although our empirical results suggest robust performance in standard settings.

We plan to address these aspects in future work to enhance the versatility and autonomy of our CIL framework.

\section{Prototype-based Classifier and Alignment}
\label{app:alignment}

Following the established settings in MOS~\cite{sun2025mos} and TUNA~\cite{wang2025integrating}, we adopt a non-parametric strategy to construct the classification head. Let $\mathcal{D}_c$ denote the set of training samples available for a specific class $c$ in the current task. We calculate the class prototype $\mathbf{p}_c$ by computing the centroid of the feature representations for all samples in this set:
\begin{equation}
\label{eq:app_prototype_calc}
    \mathbf{p}_c = \frac{1}{|\mathcal{D}_c|} \sum_{\mathbf{x} \in \mathcal{D}_c} \phi(\mathbf{x}; \mathcal{A}_t),
\end{equation}
where $\phi(\mathbf{x}; \mathcal{A}_t)$ represents the feature extraction function of the backbone model parameterized by the current task's adapter $\mathcal{A}_t$, and $|\mathcal{D}_c|$ is the number of samples in class $c$. The corresponding classifier weight vector $\mathbf{w}_c$ for class $c$ is subsequently derived by applying $L_2$ normalization to the prototype: $\mathbf{w}_c = \mathbf{p}_c / \|\mathbf{p}_c\|_2$.

To address the challenge of \textit{subspace misalignment}, where historical prototypes diverge from the current feature space, we utilize the feature distribution modeling method from TUNA~\cite{wang2025integrating} and MOS~\cite{sun2025mos}. Specifically, we model the feature distribution of each previously learned class $k$ as a multivariate Gaussian distribution $\mathcal{N}(\boldsymbol{\mu}_k, \boldsymbol{\Sigma}_k)$. During the training of the current task $t$, we synthesize pseudo-features by sampling from these stored distributions. These features are then used to update the prototypes of old classes, resulting in a fully aligned global classifier matrix $W_t = [\mathbf{w}_1, \ldots, \mathbf{w}_{|\mathcal{Y}_t|}]$, where $|\mathcal{Y}_t|$ denotes the total number of accumulated classes.

\section{Theoretical Analysis for Task-Specific and Global Adapters} \label{proof_modified}

We first reintroduce the classical PAC-Bayes adapted from~\cite{alquier2016properties, wang2024forgetting} as the Lemma.

\begin{lemma}[Adapted from~\cite{alquier2016properties}, Thm 4.1] \label{th:naive pac bayes}
    Let $\mathcal{D}=(x_1,...,x_m)$ be an iid set sampled from the law $\mu$.
  For any data-free prior $P$, for any loss function $\ell$ bounded by $K$, any $\lambda>0,\delta\in [0,1]$, one has with probability $1-\delta$ for any posterior $Q\in\mathcal{M}_1(\mathcal{H})$:
  \begin{equation}
  \mathbb{E}_{h\sim Q}\mathbb{E}_{x\sim \mu}[\ell(h,x)] \leq \frac{1}{m} \sum_{i=1}^m \mathbb{E}_{h\sim Q}[\ell(h,x_i)] + \frac{\operatorname{KL}(Q\| P) + \log(1/\delta)}{\lambda} + \frac{\lambda K^2}{2m}, 
  \end{equation}
  where $\mathcal{M}_1(\mathcal{H})$ denotes the set of all probability distributions on $\mathcal{H}$.
\end{lemma}

\subsection{Setup and Definitions}

To adapt the PAC-Bayes framework to our method, we formalize the structural decomposition of the parameter space and the distributions used in our algorithm.

\begin{definition}[Factorized Parameter Space and Distributions] \label{def: factorized dists}
    Let the hypothesis space be decomposable into $\mathcal{H} = \mathcal{H}_s \times \mathcal{H}_g$, where $\mathcal{H}_s$ represents the task-specific parameters (special adapter) and $\mathcal{H}_g$ represents the global parameters (global adapter). 
    For a specific task $t$, we define the distributions as follows:
    \begin{itemize}
        \item \textbf{Prior $P_t$}: Constructed before observing $\mathcal{D}_t$, defined as the product measure of a fixed pre-trained initialization $P_{\text{init}}$ and the previous global posterior $Q_{t-1}^g$:
        \begin{equation}
             P_t(h_s, h_g) = P_{\text{init}}(h_s) \otimes Q_{t-1}^g(h_g).
        \end{equation}
        
        \item \textbf{Posterior $Q_t$}: The variational family is restricted to product measures (Mean-Field approximation) of the learned task-specific adapter $Q_t^s$ and the updated global adapter $Q_t^g$:
        \begin{equation}
            Q_t(h_s, h_g) = Q_t^s(h_s) \otimes Q_t^g(h_g).
        \end{equation}
    \end{itemize}
    Crucially, $P_{\text{init}}$ is data-independent, and $Q_{t-1}^g$ depends only on past tasks $\mathcal{D}_{1:t-1}$, making $P_t$ a valid data-free prior relative to $\mathcal{D}_t$.
\end{definition}

Based on this setup, we derive the generalization bound for the current task $t$.

\begin{theorem}
  \label{th: split pac bayes}
  Let $\mathcal{D}_t=(x_1,...,x_{m_t})$ be an iid set sampled from the distribution $\mu_t$ for task $t$. 
  For any loss function $\ell$ bounded by $K$, any $\lambda>0, \delta\in [0,1]$, following the setting defined in Definition \ref{def: factorized dists}, the following inequality holds with probability $1-\delta$:
  \begin{align}
    \mathbb{E}_{h \sim Q_{t}}\left[ \mathbb{E}_{x \sim \mu_t}[\ell(h,x)] \right] 
    \leq & \frac{1}{m_t} \sum_{j=1}^{m_t} \mathbb{E}_{h \sim Q_{t}}\left[ \ell(h,x_j) \right]
    + \frac{\operatorname{KL}(Q_t^s \| P_{\text{init}})}{\lambda} + \frac{\operatorname{KL}(Q_t^g \| Q_{t-1}^g)}{\lambda} \nonumber\\
    & + \frac{\log(1/\delta)}{\lambda} + \frac{\lambda K^2}{2m_t}.
  \end{align}
\end{theorem}

\begin{proof}
    We apply Lemma \ref{th:naive pac bayes} to the task $t$. Let $P = P_t$ and $Q = Q_t$. Since $P_t$ is independent of the current samples $\mathcal{D}_t$, the standard PAC-Bayes bound holds with probability $1-\delta$:
    \begin{equation} \label{eq:base_bound}
    \mathbb{E}_{h\sim Q_t}\mathbb{E}_{x\sim \mu}[\ell(h,x)] \leq \frac{1}{m} \sum_{i=1}^m \mathbb{E}_{h\sim Q_t}[\ell(h,x_i)] + \frac{\operatorname{KL}(Q_t \| P_t)}{\lambda} + \frac{\log(1/\delta)}{\lambda} + \frac{\lambda K^2}{2m_t}.
    \end{equation}
    
    The key step is to decompose the divergence term $\operatorname{KL}(Q_t \| P_t)$. Following Definition \ref{def: factorized dists}, both the prior and the posterior are product measures on the space $\mathcal{H}_s \times \mathcal{H}_g$. By the additivity property of the Kullback-Leibler divergence for independent distributions, we have:
    \begin{align}
        \operatorname{KL}(Q_t \| P_t) &= \operatorname{KL}(Q_t^s \otimes Q_t^g \| P_{\text{init}} \otimes Q_{t-1}^g) \nonumber \\
        &= \int_{\mathcal{H}_s}\int_{\mathcal{H}_g} Q_t^s(h_s)Q_t^g(h_g) \log \left( \frac{Q_t^s(h_s)Q_t^g(h_g)}{P_{\text{init}}(h_s)Q_{t-1}^g(h_g)} \right) d h_g d h_s \nonumber \\
        &= \int_{\mathcal{H}_s} Q_t^s(h_s) \log \frac{Q_t^s(h_s)}{P_{\text{init}}(h_s)} dh_s \cdot \underbrace{\int_{\mathcal{H}_g} Q_t^g(h_g) dh_g}_{=1} \nonumber \\
        &\quad + \int_{\mathcal{H}_g} Q_t^g(h_g) \log \frac{Q_t^g(h_g)}{Q_{t-1}^g(h_g)} dh_g \cdot \underbrace{\int_{\mathcal{H}_s} Q_t^s(h_s) dh_s}_{=1} \nonumber \\
        &= \operatorname{KL}(Q_t^s \| P_{\text{init}}) + \operatorname{KL}(Q_t^g \| Q_{t-1}^g).
    \end{align}
    
    Substituting this result back into Eq. (\ref{eq:base_bound}) and separating the terms concludes the proof.
\end{proof}

\begin{remark}
    Theorem \ref{th: split pac bayes} theoretically grounds the three objectives of our method to achieve a balance between plasticity and stability:
    \begin{enumerate}
        \item The empirical risk term $\frac{1}{m_t} \sum_{j=1}^{m_t} \mathbb{E}_{h \sim Q_{t}}\left[ \ell(h,z_j) \right]$ encourages the adapter to learn from current data, ensuring the model's \textbf{plasticity}.
        \item The term $\frac{\operatorname{KL}(Q_t^s \| P_{\text{init}})}{\lambda}$ regularizes the task-specific adapter towards the robust pre-trained initialization. This not only prevents overfitting but also helps align the current model distribution with potential future distributions, contributing to the model's \textbf{stability}.
        \item The term $\frac{\operatorname{KL}(Q_t^g \| Q_{t-1}^g)}{\lambda}$ constrains the global adapter to remain close to the knowledge acquired from previous tasks. This explicitly ensures \textbf{stability} by mitigating catastrophic forgetting.
    \end{enumerate}
\end{remark}

\section{Proof of Relationship between $\theta_{t}$, $\theta_{t}^{*}$ and $\Delta\theta$}
\label{app:theta_diff_proof}

From the definition of the stability constraint parameter $\Delta\theta$ in Eq. \eqref{eq:constraint} and the fusion strategy in Eq. \eqref{eq:fusion}, we first expand $\Delta\theta$:
\begin{equation}
\Delta\theta = \theta_{t}^{*} - \theta_{t-1}^{*} + \theta_{t} - \theta_{p}.
\end{equation}
Substitute $\theta_{t}^{*}$ with the explicit fusion formula $\beta_{t}\theta_{p} + \beta_{t}\theta_{t-1}^{*} + \left(1 - 2\beta_{t}\right)\theta_{t}$:
\begin{equation}
\begin{aligned}
\Delta\theta &= \left[\beta_{t}\theta_{p} + \beta_{t}\theta_{t-1}^{*} + (1 - 2\beta_{t})\theta_{t}\right] - \theta_{t-1}^{*} + \theta_{t} - \theta_{p} \\
&= (\beta_{t}-1)\theta_{p} + (\beta_{t}-1)\theta_{t-1}^{*} + (1 - 2\beta_{t} + 1)\theta_{t} \\
&= (\beta_{t}-1)\theta_{p} + (\beta_{t}-1)\theta_{t-1}^{*} + (2 - 2\beta_{t})\theta_{t}.
\end{aligned}
\end{equation}
By factoring out $(\beta_{t}-1)$, we obtain:
\begin{equation}
\label{eq:delta_theta_factor}
\Delta\theta = (\beta_{t}-1)(\theta_{p} + \theta_{t-1}^{*} - 2\theta_{t}).
\end{equation}

Next, we calculate the difference between the updated global adapter and the task-specific adapter, $\theta_{t}^{*} - \theta_{t}$:
\begin{equation}
\theta_{t}^{*} - \theta_{t} = \left[\beta_{t}\theta_{p} + \beta_{t}\theta_{t-1}^{*} + (1 - 2\beta_{t})\theta_{t}\right] - \theta_{t}.
\end{equation}
Simplify the equation by grouping terms:
\begin{equation}
\begin{aligned}
\theta_{t}^{*} - \theta_{t} &= \beta_{t}\theta_{p} + \beta_{t}\theta_{t-1}^{*} + (1 - 2\beta_{t} - 1)\theta_{t} \\
&= \beta_{t}\theta_{p} + \beta_{t}\theta_{t-1}^{*} - 2\beta_{t}\theta_{t}.
\end{aligned}
\end{equation}
Factor out $\beta_{t}$:
\begin{equation}
\label{eq:diff_factor}
\theta_{t}^{*} - \theta_{t} = \beta_{t}(\theta_{p} + \theta_{t-1}^{*} - 2\theta_{t}).
\end{equation}

Finally, comparing Eq. \eqref{eq:delta_theta_factor} and Eq. \eqref{eq:diff_factor}, we substitute the common term $(\theta_{p} + \theta_{t-1}^{*} - 2\theta_{t}) = \frac{\Delta\theta}{\beta_{t}-1}$ into Eq. \eqref{eq:diff_factor}:
\begin{equation}
\theta_{t}^{*} - \theta_{t} = \beta_{t} \left( \frac{\Delta\theta}{\beta_{t}-1} \right) = \frac{\beta_{t}}{\beta_{t}-1}\Delta\theta.
\end{equation}

\section{Detailed Deduction of the Optimal Fusion Coefficient $\beta_t$}
\label{app:beta_proof}

In this section, we provide the step-by-step deduction of the optimal fusion coefficient $\beta_t$ presented in Eq. \eqref{eq:optimal_beta}. We start with the unified optimization objective defined in Eq. \eqref{eq:optimization_problem}, which aims to minimize the performance gap and parameter shift under the stability constraint.

The Lagrangian function $F$ is formulated as:
\begin{align}
    F &= \underbrace{\mathcal{L}(\theta_{t}^{*}) - \mathcal{L}(\theta_{t})}_{\text{Loss Difference}} + \underbrace{\frac{1}{2}\left(\theta_{t}^{*} - \theta_{t-1}^{*} + \theta_{t} - \theta_{p}\right)^{2}}_{\text{Parameter Shift}} \nonumber \\
    &+ \lambda\left(\Delta\theta + \theta_{t-1}^{*} + \theta_{p} - \theta_{t} - \theta_{t}^{*}\right).
\end{align}

First, we apply the second-order Taylor expansion to the loss difference and substitute the definitions derived in Appendix \ref{app:theta_diff_proof}. Specifically, we use the relationship $\theta_{t}^{*} - \theta_{t} = \frac{\beta_{t}}{\beta_{t}-1}\Delta\theta$. The Lagrangian can be rewritten in terms of $\beta_t$ and $\Delta\theta$:
\begin{align} \label{eq:lagrangian_expanded}
    F \approx & \mathcal{L}^{\prime}(\theta_{t})\left(\frac{\beta_{t}}{\beta_{t}-1}\Delta\theta\right) + \frac{\mathcal{L}^{\prime\prime}(\theta_{t})}{2}\left(\frac{\beta_{t}}{\beta_{t}-1}\Delta\theta\right)^{2} + \frac{1}{2}\Delta\theta^{2} \nonumber \\
    &+ \lambda\left[\Delta\theta + (1-\beta_{t})(\theta_{p} + \theta_{t-1}^{*} - 2\theta_{t})\right].
\end{align}

To find the optimal $\beta_t$, we calculate the partial derivatives of $F$ with respect to $\beta_t$ and $\Delta\theta$, and set them to zero.

\textbf{1. Derivative with respect to $\beta_t$:}
\begin{align} \label{eq:partial_beta_raw}
    \frac{\partial F}{\partial\beta_{t}} = &-\frac{1}{(\beta_{t}-1)^{2}}\mathcal{L}^{\prime}(\theta_{t})\Delta\theta - \frac{\beta_{t}}{(\beta_{t}-1)^{3}}\mathcal{L}^{\prime\prime}(\theta_{t})\Delta\theta^{2} \nonumber \\
    &- \lambda(\theta_{p} + \theta_{t-1}^{*} - 2\theta_{t}) = 0.
\end{align}

\textbf{2. Derivative with respect to $\Delta\theta$:}
\begin{equation} \label{eq:partial_delta}
    \frac{\partial F}{\partial \Delta\theta} = \frac{\beta_{t}}{\beta_{t}-1}\mathcal{L}^{\prime}(\theta_{t}) + \frac{\beta_{t}^{2}}{(\beta_{t}-1)^{2}}\mathcal{L}^{\prime\prime}(\theta_{t})\Delta\theta + \Delta\theta + \lambda = 0.
\end{equation}

From the constraint equation $\Delta\theta = (\beta_t - 1)(\theta_p + \theta_{t-1}^* - 2\theta_t)$, we can express the term involving parameters as:
\begin{equation} \label{eq:param_sub}
    \theta_{p} + \theta_{t-1}^{*} - 2\theta_{t} = \frac{1}{\beta_{t}-1}\Delta\theta.
\end{equation}

Now, we substitute $\lambda$ derived from Eq. \eqref{eq:partial_delta} and the relationship from Eq. \eqref{eq:param_sub} into Eq. \eqref{eq:partial_beta_raw}. This eliminates $\lambda$:
\begin{align}
    0 = &-\frac{1}{(\beta_{t}-1)^{2}}\mathcal{L}^{\prime}(\theta_{t})\Delta\theta - \frac{\beta_{t}}{(\beta_{t}-1)^{3}}\mathcal{L}^{\prime\prime}(\theta_{t})\Delta\theta^{2} \nonumber \\
    &- \left[ -\frac{\beta_{t}}{\beta_{t}-1}\mathcal{L}^{\prime}(\theta_{t}) - \frac{\beta_{t}^{2}}{(\beta_{t}-1)^{2}}\mathcal{L}^{\prime\prime}(\theta_{t})\Delta\theta - \Delta\theta \right] \left( \frac{1}{\beta_{t}-1}\Delta\theta \right).
\end{align}

Expanding and grouping the terms:
\begin{align}
    0 = &-\frac{1}{(\beta_{t}-1)^{2}}\mathcal{L}^{\prime}(\theta_{t})\Delta\theta - \frac{\beta_{t}}{(\beta_{t}-1)^{3}}\mathcal{L}^{\prime\prime}(\theta_{t})\Delta\theta^{2} \nonumber \\
    &+ \frac{\beta_{t}}{(\beta_{t}-1)^{2}}\mathcal{L}^{\prime}(\theta_{t})\Delta\theta + \frac{\beta_{t}^{2}}{(\beta_{t}-1)^{3}}\mathcal{L}^{\prime\prime}(\theta_{t})\Delta\theta^{2} + \frac{1}{\beta_{t}-1}\Delta\theta^{2}.
\end{align}

We can combine the coefficients for $\mathcal{L}^{\prime}(\theta_{t})$ and $\mathcal{L}^{\prime\prime}(\theta_{t})$:
\begin{align}
    0 = &\left[ \frac{\beta_{t}-1}{(\beta_{t}-1)^{2}} \right] \mathcal{L}^{\prime}(\theta_{t})\Delta\theta + \left[ \frac{\beta_{t}^{2}-\beta_{t}}{(\beta_{t}-1)^{3}} \right] \mathcal{L}^{\prime\prime}(\theta_{t})\Delta\theta^{2} \nonumber \\
    &+ \frac{1}{\beta_{t}-1}\Delta\theta^{2}.
\end{align}

Simplifying the fractions, we obtain:
\begin{equation}
    0 = \frac{\Delta\theta}{\beta_{t}-1} \left[ \mathcal{L}^{\prime}(\theta_{t}) + \frac{\beta_{t}}{\beta_{t}-1}\mathcal{L}^{\prime\prime}(\theta_{t})\Delta\theta + \Delta\theta \right].
\end{equation}

By observation, we can find one trivial solution where $\Delta\theta=0$. This implies $\theta_{t}^{*} - \theta_{t-1}^{*} + \theta_{t} - \theta_{p} = 0$, representing that the model has not acquired new knowledge or shifted its parameters effectively. Obviously, this is not the global optimal solution we seek. 

Furthermore, since the operational range of $\beta_t$ is constrained to $(0, 0.5)$, we have $\beta_t - 1 \neq 0$. Consequently, to find the non-trivial optimal solution, we can safely divide the equation by $\frac{\Delta\theta}{\beta_t - 1}$:
\begin{equation}
    0 = \mathcal{L}^{\prime}(\theta_{t}) + \frac{\beta_{t}}{\beta_{t}-1}\mathcal{L}^{\prime\prime}(\theta_{t})\Delta\theta + \Delta\theta.
\end{equation}

Next, we substitute $\Delta\theta$ back using the relationship $\Delta\theta = (\beta_t - 1)(\theta_p + \theta_{t-1}^* - 2\theta_t)$:
\begin{align}
    0 &= \mathcal{L}^{\prime}(\theta_{t}) + \frac{\beta_{t}}{\beta_{t}-1}\mathcal{L}^{\prime\prime}(\theta_{t})\left[(\beta_t - 1)(\theta_p + \theta_{t-1}^* - 2\theta_t)\right] \nonumber \\
    &\quad + (\beta_t - 1)(\theta_p + \theta_{t-1}^* - 2\theta_t).
\end{align}

Simplifying further:
\begin{align}
    0 &= \mathcal{L}^{\prime}(\theta_{t}) + \beta_{t}\mathcal{L}^{\prime\prime}(\theta_{t})(\theta_p + \theta_{t-1}^* - 2\theta_t) \nonumber \\
    &\quad + (\beta_t - 1)(\theta_p + \theta_{t-1}^* - 2\theta_t).
\end{align}

Rearranging terms to isolate $\beta_t$:
\begin{align}
    0 &= \mathcal{L}^{\prime}(\theta_{t}) - (\theta_p + \theta_{t-1}^* - 2\theta_t) \nonumber \\
    &\quad + \beta_{t}(\theta_p + \theta_{t-1}^* - 2\theta_t)\left[ \mathcal{L}^{\prime\prime}(\theta_{t}) + 1 \right].
\end{align}

Finally, solving for $\beta_t$, we obtain the optimal fusion coefficient:
\begin{equation}
    \beta_{t} = \frac{(\theta_{p} + \theta_{t-1}^{*} - 2\theta_{t}) - \mathcal{L}^{\prime}(\theta_{t})}{(\theta_{p} + \theta_{t-1}^{*} - 2\theta_{t})(\mathcal{L}^{\prime\prime}(\theta_{t}) + 1)}.
\end{equation}

\subsection{Verification of Constraint Satisfaction}

According to Eq. \eqref{eq:delta_theta_factor} derived in Appendix \ref{app:theta_diff_proof}, we have established the relationship:
\begin{equation}
\Delta\theta = (\beta_{t}-1)(\theta_{p} + \theta_{t-1}^{*} - 2\theta_{t}).
\end{equation}
Thus, we can verify that our proposed solution satisfies the \textit{s.t.} constraint (Eq. \eqref{eq:constraint}) through the following derivation:
\begin{align}
\Delta\theta+\theta_{t-1}^{*}+\theta_{p}-\theta_{t}-\theta_{t}^{*} &= (\beta_{t}-1)(\theta_{p}+\theta_{t-1}^{*}-2\theta_{t})+\theta_{t-1}^{*}+\theta_{p}-\theta_{t}-\theta_{t}^{*} \nonumber \\
&= \beta_{t}\theta_{p}+\beta_{t}\theta_{t-1}^{*}-2\beta_{t}\theta_{t}-\theta_{p}-\theta_{t-1}^{*}+2\theta_{t}+\theta_{t-1}^{*}+\theta_{p}-\theta_{t}-\theta_{t}^* \nonumber \\
&= \beta_{t}\theta_{p}+\beta_{t}\theta_{t-1}^{*}-2\beta_{t}\theta_{t}+\theta_{t}-\theta_{t}^{*} \nonumber \\
&= \beta_{t}\theta_{p}+\beta_{t}\theta_{t-1}^{*}-2\beta_{t}\theta_{t}+\theta_{t}-[\beta_{t}\theta_{p}+\beta_{t}\theta_{t-1}^{*}+(1-2\beta_{t})\theta_{t}] \nonumber \\
&= \beta_{t}\theta_{p}+\beta_{t}\theta_{t-1}^{*}-2\beta_{t}\theta_{t}+\theta_{t}-\beta_{t}\theta_{p}-\beta_{t}\theta_{t-1}^*-\theta_{t}+2\beta_{t}\theta_{t} \nonumber \\
&= 0.
\end{align}
This confirms that the fusion mechanism strictly adheres to the stability constraint defined in our optimization problem.

\section{Algorithm Details} \label{algorithm}

In this section, we provide the detailed algorithmic procedure of our proposed method. As outlined in Section \ref{sec:methodology}, our framework proceeds sequentially through each task $t \in \{1, \ldots, T\}$. The core mechanism involves training a disposable task-specific adapter initialized with historical knowledge, and subsequently fusing it into a persistent global adapter using a theoretically derived dynamic coefficient $\beta_t$.

The complete training process is summarized in Algorithm \ref{alg:dynamic_fusion}.

\begin{algorithm}[h]
    \caption{Dynamical Fusion with Task-Specific Adapters for CIL}
    \label{alg:dynamic_fusion}
    \textbf{KwIn: }{Stream of tasks $\{\mathcal{D}_1, \ldots, \mathcal{D}_T\}$, Pre-trained backbone $\phi$, Adapter initialization scheme $\theta_{init}$, Learning rate $\eta$, Scaling factor $\alpha$, Epochs $E$.} \\
    \textbf{KwOut: }{Final global adapter $\theta_T^*$.} \\
    \textbf{Initialize:} {Global adapter $\theta_0^* \leftarrow \theta_{init}$, Historical average $\theta_{avg} \leftarrow \theta_{init}$.} \\
    \textbf{For} {$t$ = 1, ..., $T$} \\
    {
        1. Set prior $\theta_p \leftarrow \theta_{avg}$ and initialize task-specific adapter: $\theta_t \leftarrow \theta_p$. \\
        \textbf{For} {$epoch$ = 1, ..., $E$} \\
        {
            \textbf{For }{each batch $(\mathbf{x}, y)$ in $\mathcal{D}_t$} \\
            {
                2. Compute Cross-Entropy loss $\mathcal{L}(\theta_t; \mathbf{x}, y)$. \\
                3. Update task parameters: $\theta_t \leftarrow \theta_t - \eta \nabla_{\theta_t} \mathcal{L}$. \\
            }
        }
        4. Compute first-order gradient $\mathcal{L}'(\theta_t)$ and Fisher diagonal $F(\theta_t)$ on $\mathcal{D}_t$. \\
        5. Calculate statistics: $F_{\min} \leftarrow \min(F)$, $\bar{F} \leftarrow \text{mean}(F)$. \\
        6. Calculate element-wise fusion coefficient $\beta_t$ according to Eq. \eqref{eq:final_beta}: \\
        \hspace*{1em} $\beta_t = \frac{(\bar{F} - F_{\min})[(\theta_p + \theta_{t-1}^* - 2\theta_t) - \mathcal{L}^{\prime}(\theta_t)]}{(\theta_p + \theta_{t-1}^* - 2\theta_t)[\alpha F(\theta_t) - \alpha F_{\min} + 2\bar{F} - 2F_{\min}]}$. \\
        7. Clip $\beta_t$ to valid range $[0.001, 0.499]$. \\
        8. Fuse parameters to update global adapter via Eq. \eqref{eq:fusion}: \\
        \hspace*{1em} $\theta_{t}^* \leftarrow \beta_{t}\theta_{p} + \beta_{t}\theta_{t-1}^{*} + (1 - 2\beta_{t})\theta_{t}$. \\
        9. Update historical average via Eq. \eqref{eq:running_avg}: $\theta_{avg} \leftarrow \frac{t-1}{t}\theta_{avg} + \frac{1}{t}\theta_{t}$. \\
        10. \textbf{Memory Cleanup:} Discard $\theta_t$, clear statistics, and save only $\theta_t^*$ and $\theta_{avg}$ for next task. \\
    }
    \textbf{Return} $\theta_T^*$
\end{algorithm}
\section{Extended Hyperparameter Analysis}
\label{app:sensitivity}

In the main paper, we discussed the trade-off between Stability and Plasticity with varying constraint strengths $\alpha$. Here, we provide a detailed visualization of the overall \textbf{Average Accuracy} ($\bar{\mathcal{A}}$) to further substantiate the robustness of DAF.

As illustrated in Figure~\ref{fig:appendix_alpha_acc}, increasing $\alpha$ from 1.25 to 3.0 results in negligible fluctuations in $\bar{\mathcal{A}}$, with values ranging tightly between 79.63\% and 79.42\%. This stability confirms that while $\alpha$ shifts the internal balance between remembering old tasks and learning new ones, the aggregate performance of the global adapter remains remarkably consistent. Consequently, DAF demonstrates strong resilience to hyperparameter variations, alleviating the need for meticulous tuning in practical deployments.

\begin{figure*}[t!]
    \centering
    
    \begin{subfigure}[b]{0.48\textwidth}
        \centering
        \includegraphics[width=\linewidth]{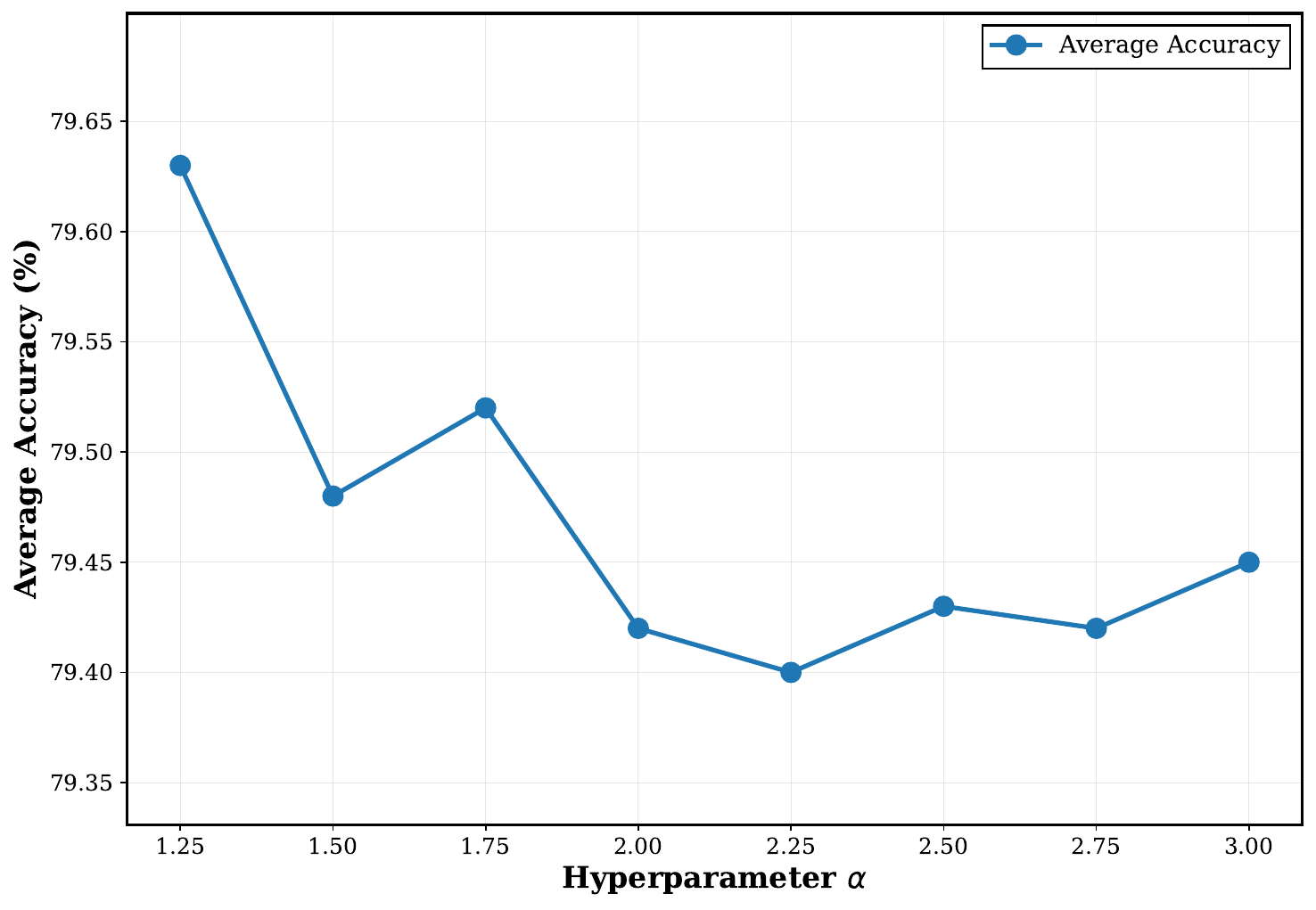} 
        \caption{Robustness to Hyperparameter $\alpha$}
        \label{fig:appendix_alpha_acc}
    \end{subfigure}
    \hfill 
    \begin{subfigure}[b]{0.48\textwidth}
        \centering
        \includegraphics[width=\linewidth]{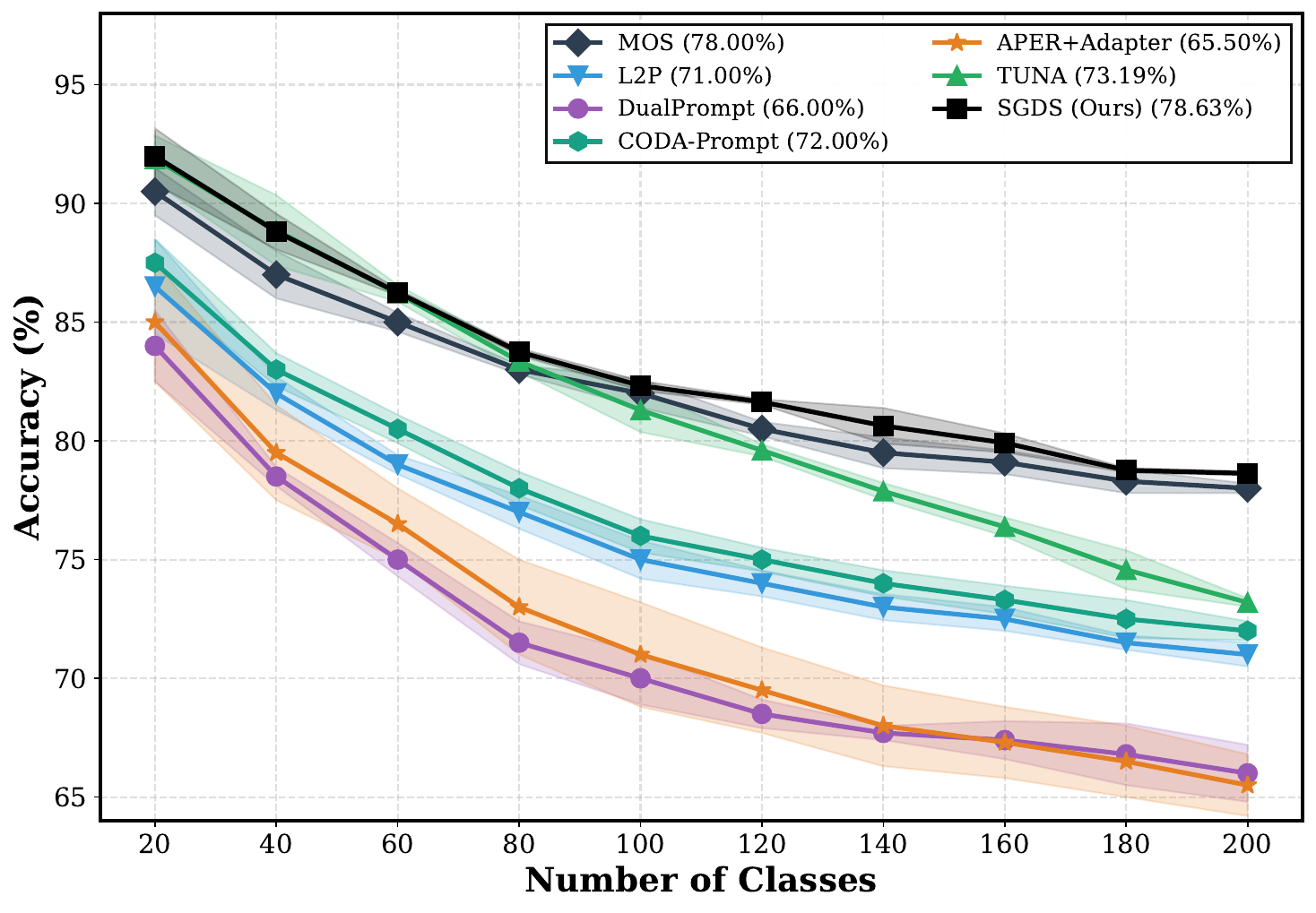} 
        \caption{Stability Across Random Seeds}
        \label{fig:multiple_runs}
    \end{subfigure}

    \caption{
        Extended stability analysis of DAF on ImageNet-R.
        \textbf{(a)} Evaluation of the Final Average Accuracy ($\bar{\mathcal{A}}$) as the constraint strength $\alpha$ varies from 1.25 to 3.0. The performance remains remarkably stable, validating that DAF is insensitive to precise hyperparameter tuning.
        \textbf{(b)} Mean Average Accuracy curves over five different random seeds. DAF consistently outperforms SOTA methods, demonstrating strong robustness to variations in class order.
    }
    \label{fig:appendix_combined_analysis}
\end{figure*}
\section{Performance Stability Across Multiple Runs}
\label{appendix:multiple_runs}

While our main experiments adhere to the standard protocol of~\cite{rebuffi2017icarl} using a fixed seed (1993), real-world scenarios often present unpredictable task sequences. To rigorously verify the consistency of DAF, we extend our evaluation on ImageNet-R by conducting five independent runs with distinct random seeds: \{1993, 1994, 1995, 1996, 1997\}. Each seed generates a unique permutation of class orders, creating diverse task compositions.

The averaged performance across these runs is presented in Figure~\ref{fig:multiple_runs}. The results demonstrate that DAF maintains a consistent and significant advantage over all competing baselines regardless of the task sequence. This stability confirms that our dynamical fusion mechanism is robust to variations in the incremental learning curriculum and does not rely on specific favorable class orderings.

\section{Enhancing DAF with Task-Specific Retrieval}
\label{appendix:retrieval_enhancement}

While DAF is designed as an efficient, retrieval-free framework utilizing a Single Global Adapter, its flexible architecture allows it to be integrated with task-aware mechanisms.
Although we evaluate TUNA using only its global adapter in the main paper, the original TUNA~\cite{wang2025integrating} also introduces a combined inference strategy.
To fully evaluate extensibility, we adopt this strategy: retrieving the most relevant task-specific adapter via an entropy-based selector and combining its prediction with the global adapter.
We compare four distinct variants: standard \textbf{DAF} and \textbf{TUNA} (Global adapter only), alongside their retrieval-enhanced counterparts, \textbf{DAF (w/ Retrieval)} and \textbf{TUNA (w/ Retrieval)}.

As illustrated in Figure~\ref{fig:appendix_curves}, incorporating a retrieval module consistently yields significant performance improvements for both frameworks across all benchmarks. Specifically, \textbf{DAF (w/ Retrieval)} achieves the highest accuracy, outperforming \textbf{TUNA (w/ Retrieval)} by a clear margin. This confirms that our dynamically fused global adapter serves as a superior foundation for ensemble-based inference compared to TUNA.

However, it is crucial to acknowledge the cost of this improvement. While retrieval mechanisms effectively boost accuracy, they introduce substantial computational overhead during inference, as the model must search through a growing pool of adapters for every input sample. Consequently, although \textbf{DAF (w/ Retrieval)} offers the best absolute performance, the standard \textbf{DAF} remains the optimal choice for real-world applications where inference speed and memory efficiency are paramount.

\begin{figure*}[t!]
    \centering
    \begin{subfigure}[b]{0.48\textwidth}
        \centering
        \includegraphics[width=\textwidth]{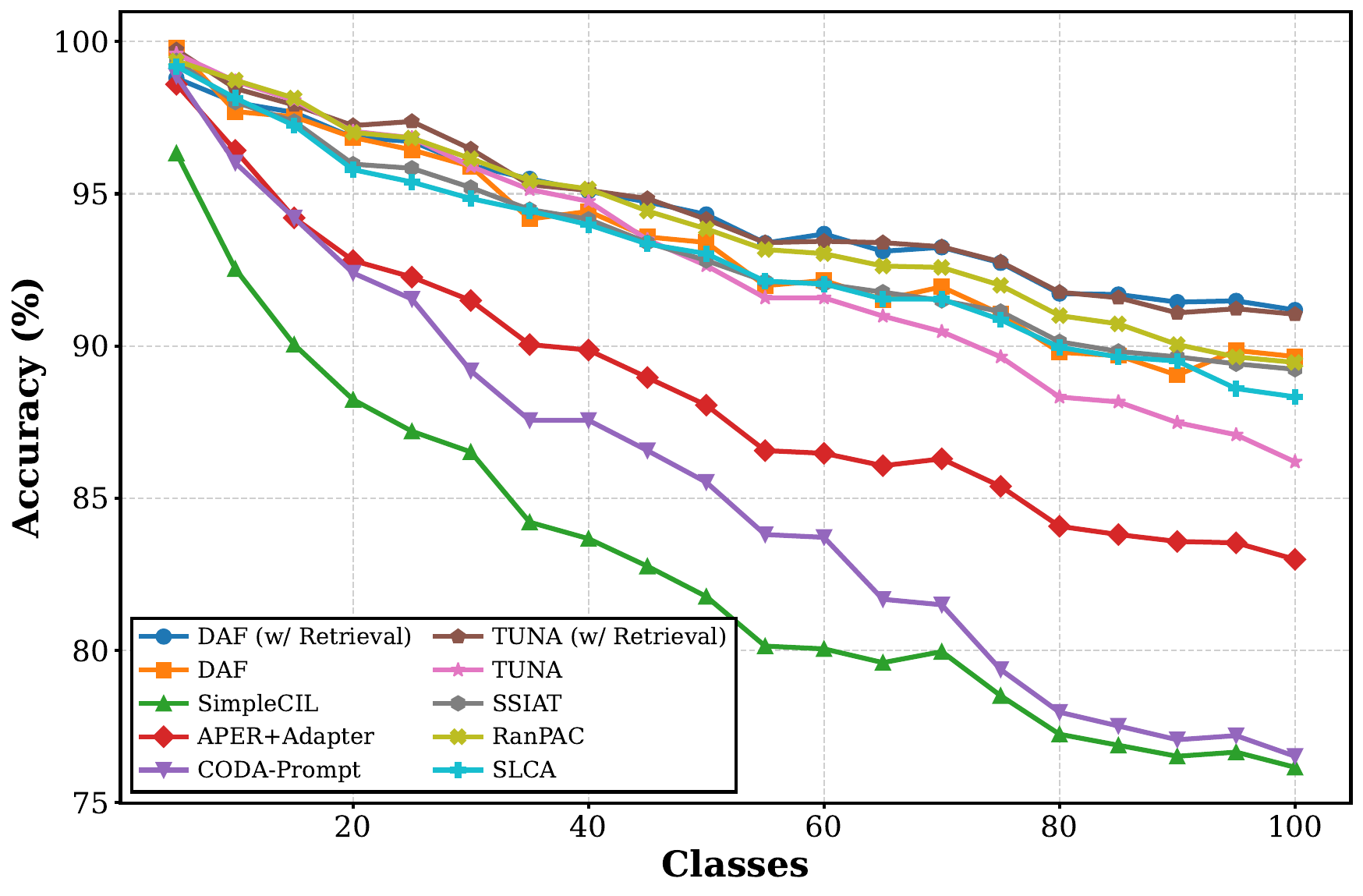}
        \caption{CIFAR-100}
        \label{subfig:curve_cifar100_1}
    \end{subfigure}
    \hfill
    \begin{subfigure}[b]{0.48\textwidth}
        \centering
        \includegraphics[width=\textwidth]{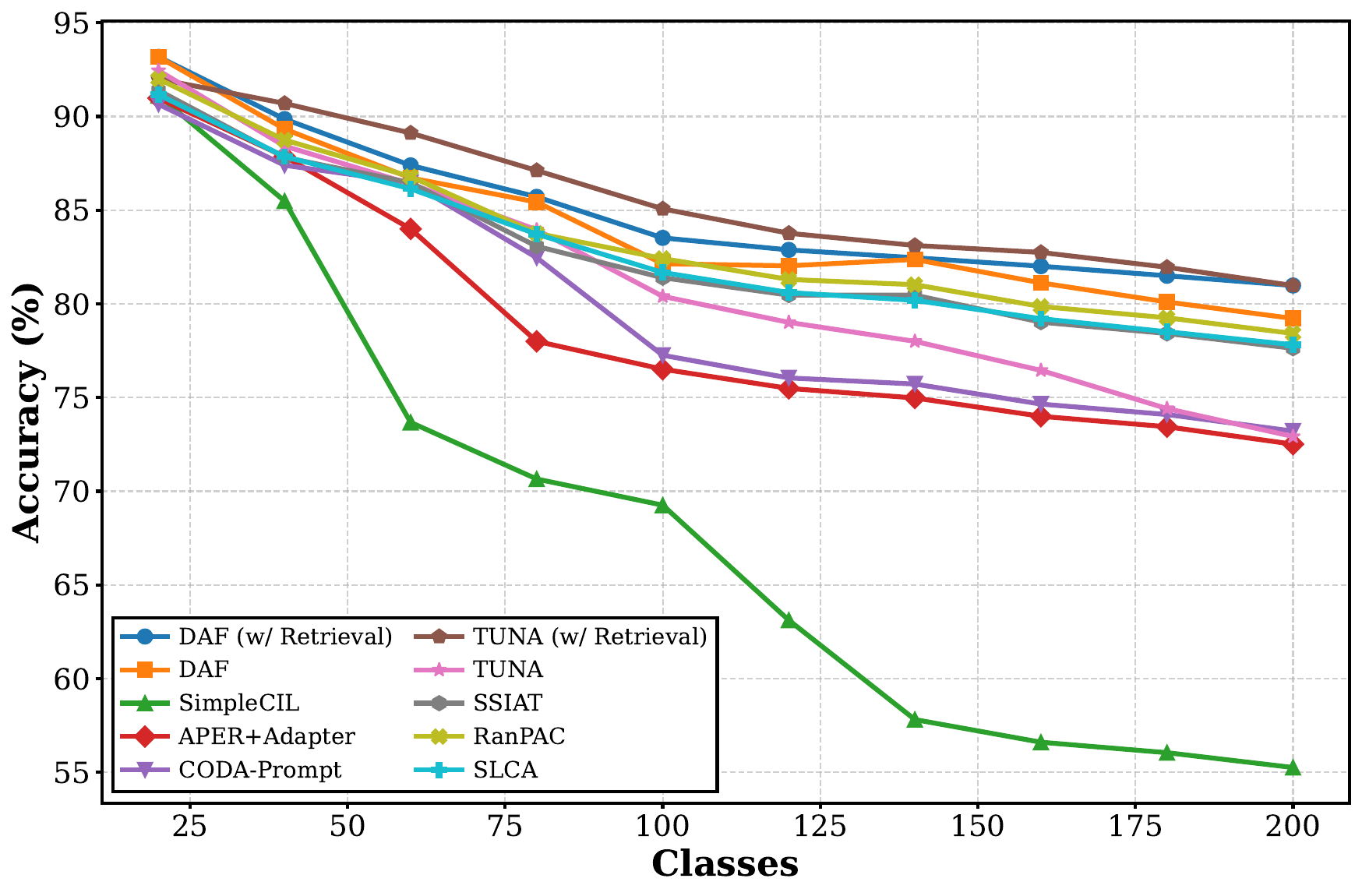}
        \caption{ImageNet-R}
        \label{subfig:curve_imagenetr_1}
    \end{subfigure}

    \vspace{0.5cm} 

    \begin{subfigure}[b]{0.48\textwidth}
        \centering
        \includegraphics[width=\textwidth]{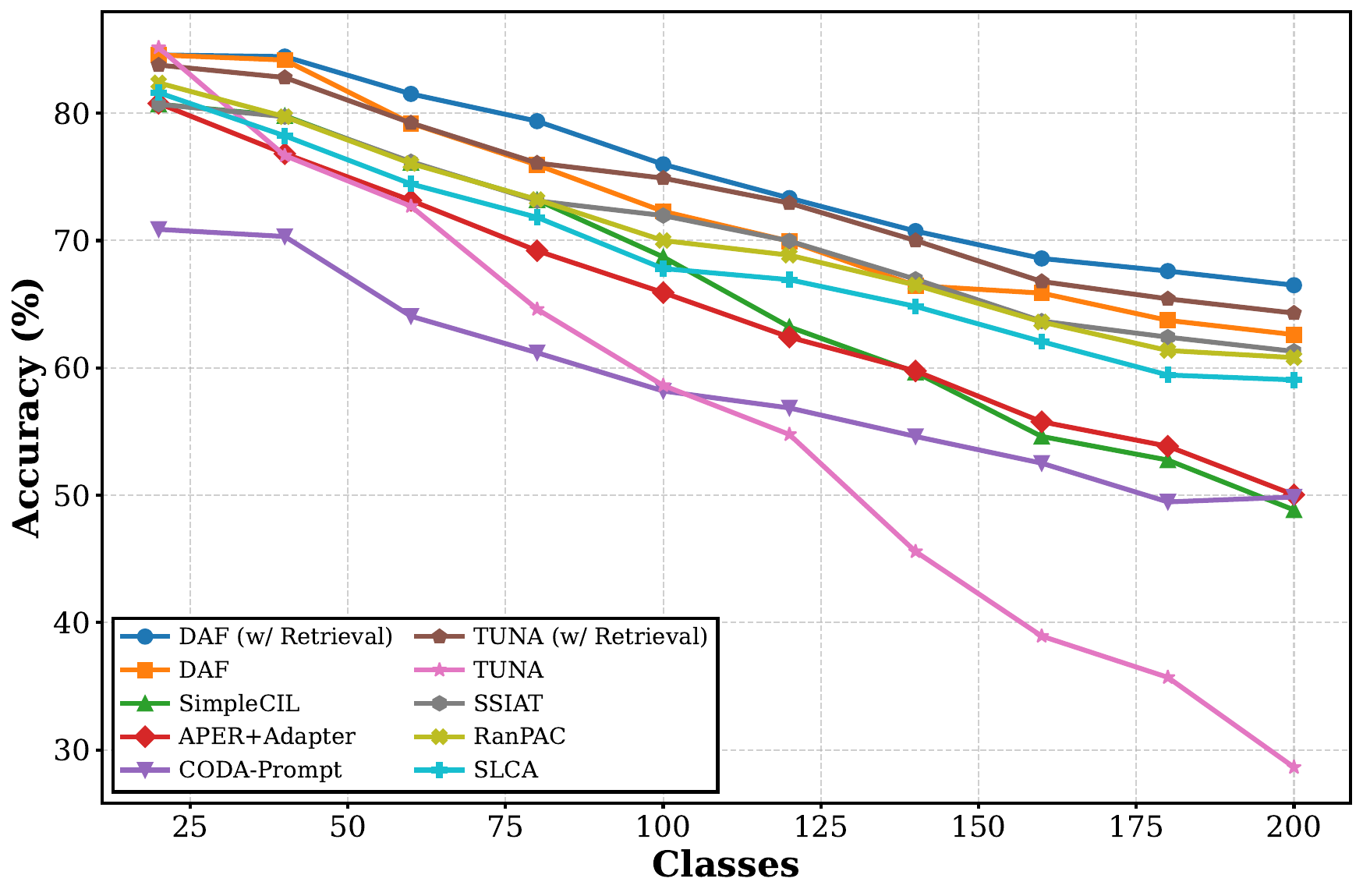}
        \caption{ImageNet-A}
        \label{subfig:curve_imageneta_1}
    \end{subfigure}
    \hfill
    \begin{subfigure}[b]{0.48\textwidth}
        \centering
        \includegraphics[width=\textwidth]{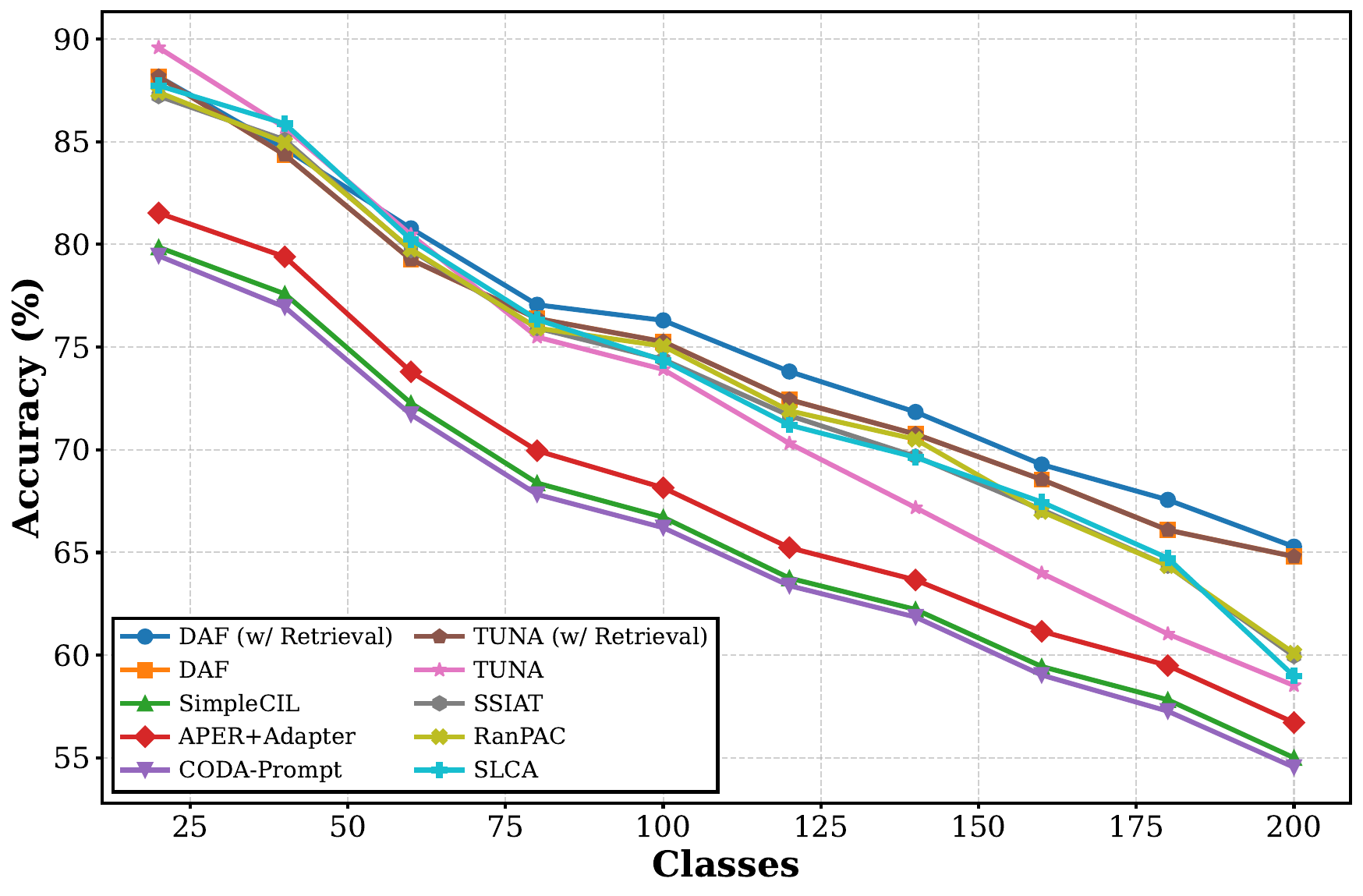}
        \caption{ObjectNet}
        \label{subfig:curve_objectnet_1}
    \end{subfigure}

    \caption{
        \textbf{Impact of Task-Specific Retrieval on DAF and TUNA.}
        We compare four variants: DAF and TUNA with and without retrieval modules.
        While integrating retrieval mechanisms significantly boosts accuracy for both methods, it comes at the expense of high inference latency.
        Notably, \textbf{DAF (w/ Retrieval)} consistently outperforms \textbf{TUNA (w/ Retrieval)}, demonstrating the superiority of our underlying fusion mechanism.
    }
    \label{fig:appendix_curves}
\end{figure*}
\end{document}